\documentclass{article}

\usepackage{geometry}

\usepackage[utf8]{inputenc} % allow utf-8 input
\usepackage[T1]{fontenc}    % use 8-bit T1 fonts

\usepackage{hyperref}       % hyperlinks
\usepackage{url}            % simple URL typesetting
\usepackage{booktabs}       % professional-quality tables
\usepackage{amsfonts}       % blackboard math symbols
\usepackage{nicefrac}       % compact symbols for 1/2, etc.
\usepackage{microtype}      % microtypography
\usepackage{xcolor}         % colors
\usepackage{comment}

\usepackage{graphicx}
\usepackage{latexsym}
\usepackage{epsfig}
\usepackage{amssymb}
\usepackage{amstext}
\usepackage{amsgen}
\usepackage{amsxtra}
\usepackage{amsgen}
\usepackage{amsthm,mathrsfs,mathtools}
\usepackage{bbm}
\usepackage{enumitem}
\setlist[enumerate]{noitemsep}
\setlist[itemize]{noitemsep}
\setlist[description]{noitemsep}
\usepackage{amsmath,color}
\usepackage{tikz}
\usepackage{comment}
\usepackage{subcaption}
\usepackage{multicol}
\usepackage[capbesideposition={right}]{floatrow}

\newcommand{\Esse}{S}

\newtheorem{thm}{Theorem}
 
\newtheorem{lem}{Lemma} 
\newtheorem{prop}{Proposition} 
\theoremstyle{definition} 
\newtheorem{defn}{Definition}
\newtheorem{hyp}{Hypothesis}

\theoremstyle{remark} 
\newtheorem{remark}{Remark}

\newcommand{\N}{\mathbb{N}}
\newcommand{\Z}{\mathbb{Z}}
\newcommand{\R}{\mathbb{R}}

\renewcommand{\epsilon}{\varepsilon} % epsilon
\renewcommand{\Im}{\operatorname{Im}}
\DeclareMathOperator{\spann}{span}

\DeclareMathOperator*\comp{\bigcirc}

\def\<{\langle}
\def\>{\rangle}

\newcommand{\email}[1]{\protect\href{mailto:#1}{#1}}

\title{Continuous Generative Neural Networks: \\ A Wavelet-Based Architecture in Function Spaces}

\author{Giovanni S.\ Alberti\thanks{MaLGa Center, Department of Mathematics,  University of Genoa,  Italy (\email{giovanni.alberti@unige.it}, \email{matteo.santacesaria@unige.it}, \email{silvia.sciutto@edu.unige.it}) \\ \textbf{Keywords}: Neural networks, Generative models, Variational autoencoders, Injective networks, Inverse problems, Wavelets, Multi-resolution analysis}
\and Matteo Santacesaria\footnotemark[1] \and Silvia Sciutto\footnotemark[1]
}

\date{}

\begin{document}

\maketitle

\begin{abstract}
In this work, we present and study Continuous Generative Neural Networks (CGNNs), namely, generative models in the continuous setting: the output of a CGNN belongs to an infinite-dimensional function space. The architecture is inspired by DCGAN, with one fully connected layer, several convolutional layers and nonlinear activation functions. In the continuous $L^2$ setting, the dimensions of the spaces of each layer are replaced by the scales of a multiresolution analysis of a compactly supported wavelet. We present conditions on the convolutional filters and on the nonlinearity that guarantee that a CGNN is injective. This theory finds applications to inverse problems, and allows for deriving Lipschitz stability estimates for (possibly nonlinear) infinite-dimensional inverse problems with unknowns belonging to the manifold generated by a CGNN. Several numerical simulations, including signal deblurring, illustrate and validate this approach.
\end{abstract}

\section{Introduction}

Inverse problems (IPs) naturally arise in various applications within imaging and signal processing, encompassing tasks like denoising, deblurring, superresolution, inpainting in computer vision, as well as computed tomography (CT), magnetic resonance imaging (MRI), positron emission tomography (PET), electrical impedance tomography (EIT) in the medical domain. Unlike traditional forward problems, they suffer from the so-called ill-posedness, namely the lack of existence, uniqueness or stability of the solution \cite{Engl}. While existence and uniqueness can be guaranteed by relaxing the inverse problem and selecting a solution with particular properties, achieving stability presents a considerably greater challenge.

A classical approach to address ill-posedness consists in utilizing a regularization functional for stabilizing the reconstruction process. In recent years, machine learning based reconstruction algorithms have become the state of the art in most imaging applications \cite{AMO19,Dimakis}. Among these algorithms, the ones combining generative models with classical iterative methods -- such as the Landweber scheme -- are very promising since they retain most of the \textit{explanaibility} provided by inverse problems theory. However, despite the impressive numerical results \cite{IPwithDL, ulyanov2018deep, AKR19, Seo1,Seo, asim2020blind, Seo2}, many theoretical questions have not been studied yet, for instance concerning stability properties of the reconstruction. One of the main contributions of this work consists in showing Lipschitz stability guarantees for infinite-dimensional IPs whose unknown belongs to the image of a specific deep generative model.

In order to obtain the stability result, we design a new class of generative models which act as a prior for the IP. In particular, we develop a generalization of a convolutional-type generator to a continuous setting, where the samples to be generated belong to an infinite-dimensional function space. This is motivated by the fact that in many IPs the unknows are physical quantities that are better modeled as functions than as vectors, e.g.\ the parameters or the solutions to partial differential equations (PDEs). As a result, working directly in an infinite-dimensional setting allows us to avoid discretization errors, a source of instabilities in ill-posed problems (see, e.g., \cite{kaipio-somersalo-2007,stuart-2010}). This aspect of the work is tightly related to the growing research area of neural networks in infinite-dimensional spaces, often motivated by the study of PDEs, which includes Neural Operators \cite{NN_fra_spazi_funzioni}, Deep-O-Nets \cite{DeepONet}, PINNS \cite{PINN} and many others. The general goal of these works is to approximate an operator between infinite-dimensional function spaces (e.g.\ the parameter-to-solution map of a PDE) with a neural network that does not depend on the discretization of the domain. 
In the context of inverse problems, there are several super-resolution method able to produce a continuous representation of an image  \cite{chen2021,dupont2021,sitzmann2020,skorokhodov2021,tancik2020,habring2022, habring2022generative}. In particular, in \cite{habring2022} the authors consider generative convolutional neural networks in function spaces, but without incorporating the concepts of strided convolution, including upscaling and downscaling, typical of discrete convolutional networks. Further, none of these approaches is based on the wavelet decomposition, the tool used in the present work, which naturally deals with continuous signals.

In this work (Section~\ref{sec:arch}), we introduce a family of continuous generative neural networks (CGNNs), mapping a finite-dimensional space into an infinite-dimensional function space. Inspired by the architecture of deep convolutional GANs (DCGANs) \cite{DCGAN}, CGNNs are obtained by composing an affine map with several (continuous) convolutional layers with nonlinear activation functions. The convolutional layers are constructed as maps between the subspaces of a multi-resolution analysis (MRA) at different scales, and naturally generalize discrete convolutions. In our continuous setting, the scale parameter plays the role of the resolution of the signal/image.  We note that wavelet analysis has been used in the design of deep learning architectures in the last decade \cite{M12,Mallat_scattering_network,anden2014deep, Deep_Haar_scattering, GNscattering}. 

The main result of this paper (Section~\ref{sec:injectivity}) is a set of sufficient conditions that the parameters of a CGNN must satisfy in order to guarantee global injectivity of the network. This result is far from trivial because in each convolutional layer the number of channels is reduced, and this has to be compensated by the higher scale in the MRA. Generative models that are not injective are of no use in solving inverse problems or inference problems, or at least it is difficult to study their performance from the theoretical point of view. In the discrete settings, some families of injective networks have been already thoroughly characterized \cite{ IRN,Layerwise_inversion,injective_relu, iUnets, Trumpets,  Universal_Joint_approximation,hagemann-neumayer-2021,hagemann2021stochastic}. Note that normalizing flows are injective by construction, yet they are maps between spaces of the same (generally large) dimension, a feature that does not necessarily help with our desired applications.

Indeed, another useful property of CGNNs is dimensionality reduction. For ill-posed inverse problems, it is well known that imposing finite-dimensional priors improves the stability and the quality of the reconstruction \cite{AV05,BFV14,BDF16,ADG16,BFV21}, also working with finitely-many measurements \cite{AS19,H19,AS21,AS22,Alberti_Santacesaria_Arroyo}. In practice, these priors are unknown  or cannot be analytically described: yet, they can be approximated by a (trained) CGNN. The second main result of this work (Section~\ref{sec:deep_land}) is that an injective CGNN allows us to transform  a possibly nonlinear ill-posed inverse problem into a Lipschitz stable one. 
Our stability estimate in Theorem~\ref{lipschitz_stab} is tightly connected to the works on compressed sensing for generative models (e.g.\ \cite{IPwithDL,berk-brugiapaglia-etal-2022}), because they both deal with stability for inverse problems under the assumption that the unknown lies in the image of a generative model. However, in our setup the model is infinite-dimensional, the forward map is possibly nonlinear, and its inverse may not be continuous even with full measurements.

As a proof-of-concept (Section~\ref{sec:num}), we show numerically the validity of CGNNs in performing signal deblurring on a class of one-dimensional smooth signals. The numerical model is obtained by training a VAE whose decoder is designed with an injective CGNN architecture. The classical Landweber iteration method is used as a baseline to compare CGNNs derived from different orthogonal wavelets and the correspondent discrete generative neural network (NN). We also provide some qualitative experiments on the expressivity of CGNNs. However, we would like to emphasize that the main contributions of this paper are theoretical, and that the main goal of our experiments is not to obtain
state of the art results, but only to compare continuous and discrete generative neural networks for a toy class of smooth signals. The application to real-world data and the combination with other methods in order to achieve competitive results are not within the scope of this paper and are left to future research.
We mention that, after the first version of this work was published online, several works have considered generative models with outputs in function spaces, see \cite{khorashadizadeh2022funknn,scherzer2023newton,hagemann2023multilevel, pidstrigach2023infinite, lim2023score,raonic2023convolutional}.

\section{Architecture of CGNNs}\label{sec:arch}
We first review the architecture of a DCGAN \cite{DCGAN}, and then present our continuous generalization. For simplicity, the analysis is done for $1$D signals, but it  can be extended to the $2$D case (see Appendix~\ref{sec:extension}).

\subsection{1D discrete generator architecture}\label{sub:1Ddiscrete}
A deep generative model can be defined as a map $G_\theta \colon \R^\Esse \to X$, where $X$ is a finite-dimensional space with $\Esse \leq \dim(X)$, constructed as the forward pass of a neural network with parameters $\theta$. Our main motivation being the use of generators in solving ill-posed inverse problems, we consider generators that allow for a dimensionality reduction, i.e.\ $\Esse \ll \dim(X)$, which will yield better stability in the reconstructions.

As a starting point for our continuous architecture, we then consider the one introduced in \cite{DCGAN}. It is a map $G\colon \R^S \to X$ (we drop the dependence on the parameters $\theta$) obtained by composing an affine fully connected (f.c.) layer and $L$ convolutional layers with nonlinear activation functions. More precisely:
\begin{multline*}
G \colon \mathbb{R}^{\Esse}  \xrightarrow[\text{f.c.}]{\Psi_1} (\mathbb{R}^{\alpha_1})^{c_1} \xrightarrow[\text{nonlin.}]{\sigma_1} (\mathbb{R}^{\alpha_1})^{c_1} \xrightarrow[\text{conv.}]{\Psi_2} (\mathbb{R}^{\alpha_2})^{c_2} \xrightarrow[\text{nonlin.}]{\sigma_2} (\mathbb{R}^{\alpha_2})^{c_2} \xrightarrow[\text{conv.}]{\Psi_3}  \cdots \\ 
\cdots \xrightarrow[\text{conv.}]{\Psi_{L}}  \mathbb{R}^{\alpha_{L}} \xrightarrow[\text{nonlin.}]{\sigma_{L}} \mathbb{R}^{\alpha_{L}}=X,
\end{multline*}
which can be summarized as
\begin{equation}\label{eq:G}
    G =  \comp_{l=L}^1 \left(\sigma_{l} \circ \Psi_l\right)
    .
\end{equation}

\begin{figure}
	\centering 
    \includegraphics[scale=1]{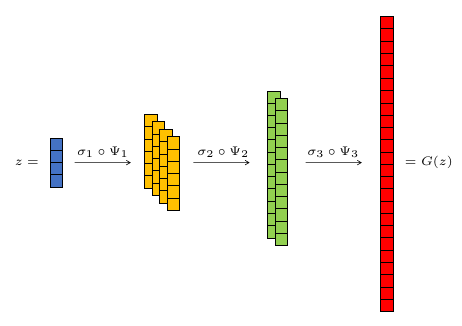}
    \caption{Example of discrete generator's architecture. The first fully connected layer maps the latent space $\R^4$ into a $4$ channel space of vectors of length $6$. Then there are two convolutional layers with stride $\frac{1}{2}$ which halve the number of channels and double the length of the vectors, until obtaining one vector of length $24$. In this example, the increase in dimensionality occurs only in the first (fully-connected) layer, while in the following convolutional layers, the dimension of the spaces remains constant, because the increase in resolution is compensated by the reduction of the number of channels.}
    \label{generator}
\end{figure}

The natural numbers $\alpha_1,\dots,\alpha_L$ are the vector sizes and represent the resolution of the signals at each layer, while $c_1,\dots,c_L$ are the number of channels at each layer. The output resolution is $\alpha_L$. Generally, one has $\alpha_1< \alpha_2 < \dots < \alpha_L$, since the resolution increases at each level. Moreover, we impose that $\alpha_{l}$ is divisible by $\alpha_{l-1}$ for every $l=2,\dots,L$. We now describe the components of $G$.

\paragraph{The nonlinearities} 
Each layer includes a pointwise nonlinearity $\sigma_{l} \colon (\mathbb{R}^{\alpha_{l}})^{c_{l}} \to (\mathbb{R}^{\alpha_{l}})^{c_{l}}$, i.e.\ a map defined as 
\begin{equation*}
\sigma_{l}(x_1,\dots,x_{\alpha_{l} \cdot c_{l}}) = (\sigma(x_1),\dots,\sigma(x_{\alpha_{l} \cdot c_{l}})), 
\end{equation*}
with $\sigma \colon \mathbb{R} \to \mathbb{R}$ nonlinear.

\paragraph{The fully connected layer} 
The first layer is $\Psi_1 := F\cdot +b$, where $F\colon\mathbb{R}^\Esse\to (\R^{\alpha_1})^{c_1}$ is a linear map and $b\in (\R^{\alpha_1})^{c_1}$ is a bias term.

\paragraph{The convolutional layers}
The convolutional layer $\Psi_{l} \colon (\mathbb{R}^{\alpha_{l - 1}})^{c_{l - 1}} \to (\mathbb{R}^{\alpha_{l}})^{c_{l}}$ represents a fractional-strided convolution with stride $s = \frac{\alpha_{l - 1}}{\alpha_{l}}$ such that $s^{-1} \in \mathbb{N}^*$, where $c_{l - 1}$ is the number of input channels and $c_{l}$ the number of output channels, with  $c_{l}<c_{l-1}$. This convolution with  stride $s$  corresponds to the transpose of the discrete convolution with stride $s^{-1}$, and is often called deconvolution. It is defined by
\begin{equation*}
    (\Psi_{l} x)_k := \sum_{i=1}^{c_{l-1}} x_i \ast_{s} t^{l}_{i,k} + {b}^{l}_k,\qquad k = 1,...,c_{l},
\end{equation*}
where $t^l_{i,k} \in \mathbb{R}^{\alpha_{l}}$ are the convolutional filters and ${b}^l_k \in \mathbb{R}^{\alpha_{l}}$  are the bias terms, for $i=1,...,c_{l-1}$ and $k=1,...,c_{l}$.
The operator $\ast_s$ is defined as
\begin{equation}
\label{deconv_discr}
(x \ast_{s} t) (n) := \sum_{m \in \Z} x(m) \hspace{0.1cm} t(n-s^{-1}m),
\end{equation}
where we extend the signals $x$ and $t$ to finitely supported sequences by defining them  zero outside their supports, i.e.\ $x,t \in c_{00}(\Z)$, where $c_{00}(\Z)$ is the space of sequences with finitely many nonzero elements. 
We refer to \ref{proof_adjoint} for more details on fractional-strided convolutions.

Note that the most significant dimensional increase occurs in the first layer, the fully connected one. Indeed, after the first layer, in the fractional-strided convolutional layers the increase of the vectors' size is compensated by the decrease of the number of channels; see Figure~\ref{generator} for an illustration. At each layer, the resolution of a signal increases thanks to a deconvolution with higher-resolution filters, as we explain in detail in \ref{proof_adjoint}. The final output is then a single high-resolution signal.

\subsection{1D CGNN architecture}
\label{sec:1D_continuous_generator_architecture}
We now describe how to reformulate this discrete architecture in the continuous setting, namely, by considering signals in $L^2(\mathbb{R}) := \{ f \colon \R \to \R \text{ Lebesgue measurable} \text{ s.t. } \int_\R f^2(x) dx < + \infty \} \mathrel{/} \sim$, where $f \sim g$ if and only if $f-g = 0$ a.e. The resolution of these continuous signals is modelled through wavelet analysis. Indeed, the higher the resolution of a signal, the finer the scale of the space to which the signal belongs. The idea to link multi-resolution analysis to neural networks is partially motivated by scattering networks \cite{Mallat_scattering_network}.

\paragraph{Basic notions of wavelet theory}
We give a brief review of concepts from wavelet analysis: in particular the definitions and the meaning of scaling function spaces and Multi-Resolution analysis in the $1$D case; for the $2$D case we refer to \ref{sec:2dwave}. See \cite{Daubechies,Hernandez,Mallat} for more details.

Given a function $\phi \in L^2(\mathbb{R})$, we define 
\begin{equation}
\label{phi_j_n}
\phi_{j,n}(x) = 2^{\frac{j}{2}} \phi (2^{j} x-n), \hspace{0.5cm} x \in \mathbb{R},
\end{equation}
for every $j,n \in \mathbb{Z}$. The integers $j$ and $n$ are the scale and the translation parameters, respectively, where the scale is proportional to the speed of the oscillations of $\phi$ (the larger $j$, the finer the scale, the faster the oscillations).

\begin{defn}
	\label{MRA}
	A \emph{Multi-Resolution Analysis (MRA)} is an increasing sequence of subspaces $\{V_j \} \subset L^2(\mathbb{R})$ defined for $j \in \mathbb{Z}$
	\begin{equation*}
	... \subset V_{-1} \subset V_0 \subset V_1 \subset ...
	\end{equation*}
	together with a function $\phi \in L^2(\mathbb{R})$ such that
	\begin{enumerate}
		\item $\displaystyle \bigcup_{j \in \mathbb{Z}} V_j$ is dense in $L^2(\mathbb{R})$ and $\displaystyle \bigcap_{j \in \mathbb{Z}} V_j = \{ 0 \}$;
		\item $f \in V_j$ if and only if $f(2^{-j} \cdot) \in V_0$;
		\item and $\{ \phi_{0,n} \}_{n \in \mathbb{Z}} = \{ \phi(\cdot-n) \}_{n \in \mathbb{Z}}$ is an orthonormal basis of $V_0$.
	\end{enumerate}
    The function $\phi$ is called \emph{scaling function} of the MRA.
\end{defn}
Intuitively, the space $V_j$ contains functions for which the finest scale is $j$.

\paragraph{The spaces} In the discrete formulation, the intermediate spaces $\mathbb{R}^{\alpha_1},\dots,\mathbb{R}^{\alpha_L}$, with $\alpha_1<\dots<\alpha_L$, describe vectors of increasing resolution. In the continuous setting, it is natural to replace these spaces by using a MRA of $L^2(\mathbb{R})$, namely, by using the spaces
\[
V_{j_1}\subset V_{j_2}\subset\cdots\subset V_{j_L},
\]
with $j_1<\dots<j_L$, representing an increasing (finite) sequence of scales. We have that $f\in V_j$ if and only if $f(2\,\cdot)\in V_{j+1}$, so that $V_{j+1}$ contains signals at a resolution that is twice that of the signals in $V_j$. Thus, the relation between the indexes $\alpha_l$ and $j_l$ is
\[
\alpha_{l}=2^{\nu} \alpha_{l-1} \iff j_{l}=\nu + j_{l-1},
\]
where $\nu \in \N$ is a free parameter, or, equivalently,
\begin{equation}
\label{stride_e_j}
j_l - j_{l-1} =\log_2\frac{\alpha_l}{\alpha_{l-1}} = \log_2 (s^{-1}),
\end{equation}
where $s=2^{-\nu}$ is the stride of the deconvolution (discussed above in the discrete setting, and defined below in the continuous setting).
Similarly to the discrete case, the intermediate spaces are $(V_{j_l})^{c_l}$ for $l=1,...,L$, with $c_1>\dots>c_L$. 
The norm in these spaces is
\begin{equation*}
  \| f \|_2^2 = \sum_{i=1}^{c_l} \| f_i \|_{L^2(\R)}^2 = \sum_{i=1}^{c_l} \int_{\R} |f_i(x)|^2 dx, \qquad f \in (V_{j_l})^{c_l}.
\end{equation*}

\paragraph{The nonlinearities} 
The nonlinearities $\sigma_l$ act on functions in $(L^2(\mathbb{R}))^{c_{l}}$ by pointwise evaluation:
\begin{equation}
    \label{sec:non_lin}
    \sigma_l(f)(x) = \sigma_l(f(x)),\qquad \text{a.e.\ $x\in \mathbb{R}$}.
\end{equation}
Note that this map is well defined if there exists $L_l$ such that $|\sigma_l(x)| \leq L_l |x|$ for every $x \in \R$. Indeed, in this case, $\sigma_l(f) \in L^2(\R)$ for $f \in L^2(\R)$. Moreover, if $f=g$ a.e., then $\sigma_{l}(f)=\sigma_{l}(g)$ a.e.\ It is worth observing that, in general, this nonlinearity does not preserve the spaces $(V_{j_{l}})^{c_{l}}$, namely, $\sigma_{l} ((V_{j_{l}})^{c_{l}})\not\subset (V_{j_{l}})^{c_{l}}$. However, in the case when the MRA is associated to the Haar wavelet, the spaces $(V_{j_{l}})^{c_{l}}$ consist of dyadic step functions, and so they are preserved by the action of $\sigma_{l}$. 

\paragraph{The fully connected layer}
The map in the first layer is given by
\begin{equation}\label{eq:Psi1}
    \Psi_1=F\cdot +b,
\end{equation}
where $F\colon\mathbb{R}^\Esse\to (V_{j_1})^{c_1}$ is a linear map and $b\in (V_{j_1})^{c_1}$.

\paragraph{The convolutional layers}  We first need to model the convolution in the continuous setting. A convolution with stride $s=2^\nu$ that maps functions from the scale $j + \nu$ to the scale $j$ with filter $g\in V_{j + \nu}\cap L^1(\R)$ can be seen as the map
\[
\cdot *_{j+\nu\to j} g \colon L^2(\mathbb{R})\to L^2(\mathbb{R}),\qquad
f *_{j+\nu\to j} g = P_{V_{j}}(P_{V_{j+\nu}}f* g),
\]
where $*\colon L^2(\R)\times L^1(\R)\to L^2(\R)$ denotes the continuous convolution and $P_V\colon L^2(\mathbb{R})\to L^2(\mathbb{R})$ denotes the orthogonal projection onto the closed subspace $V\subset L^2(\mathbb{R})$. In other words,
\begin{equation}
\label{cont_strid_conv}
\cdot *_{j+\nu\to j} g = P_{V_{j}} \circ (\,\cdot*g) \circ P_{V_{j+\nu}}.    
\end{equation}
The orthogonal projections allow us to fix the desired input and output spaces for the continuous strided convolution, analogously to the discrete case (see Figures~\ref{fig_strides} and \ref{stride_1_2_real} in Appendix~\ref{proof_adjoint}). Indeed, through a convolution with stride $s=2^\nu$, a high resolution signal (scale $j+\nu$) is mapped to a lower resolution signal (scale $j$).

We define the convolution with stride $s = 2^{- \nu}$ (also called deconvolution) with filter $g \in V_{j + \nu}\cap L^1(\R)$ as
\begin{equation}
\label{j_to_j+nu}
  \cdot *_{j\to j+\nu} g = P_{V_{j+\nu}} \circ (\,\cdot*g) \circ P_{V_{j}}\colon L^2(\mathbb{R})\to L^2(\mathbb{R}).  
\end{equation}
This is the adjoint of the convolution with stride $2^{\nu}$, which can be easily computed since projections are self-adjoint and the adjoint of a convolution with filter $g$ is a convolution with filter $\tilde{g}(x) := g(-x)$. Therefore, by renaming $\tilde{g}$ with $g$, we obtain \eqref{j_to_j+nu}.

We are now able to model a convolutional layer. The $l$-th layer of a CGNN, for $l\geq 2$, is
\[
\sigma_l\circ\bar\Psi_l \colon (L^2(\mathbb{R}))^{c_{l-1}}\to (L^2(\mathbb{R}))^{c_{l}},
\]
where $\sigma_{l}$ is the nonlinearity defined above and $\bar\Psi_l$ are the convolutions with stride $2^{-\nu}$. In view of the above discussion, and of the discrete counterpart explained in Section~\ref{sub:1Ddiscrete}, we define
\begin{equation*}
  \bar\Psi_l = P_{(V_{j_{l}})^{c_{l}}}\circ \Psi_{l} \circ P_{(V_{j_{l-1}})^{c_{l-1}}},
\end{equation*}
where the convolution $ \Psi_l \colon (L^2(\mathbb{R}))^{c_{l-1}}\to (L^2(\mathbb{R}))^{c_{l}}$ is given by
\begin{equation}
\label{psi_ell}
  (\Psi_{l} (x))_k := \sum_{i=1}^{c_{l-1}} x_i * t^{l}_{i,k} + {b}^{l}_k,\qquad k = 1,...,c_{l},
\end{equation}
with filters $t^{l}_{i,k}\in V_{j_{l}}\cap L^1(\R)$ and biases ${b}^{l}_k\in V_{j_{l}}$. 

\paragraph{Summing up} 
Altogether, the full architecture in the continuous setting may be written as
\begin{multline*}
G \colon \mathbb{R}^{\Esse}  \xrightarrow[\text{f.c.}]{\Psi_1} (V_{j_1})^{c_1} \xrightarrow[\text{nonlin.}]{\sigma_1} (L^2(\mathbb{R}))^{c_1} \xrightarrow[\text{proj.}]{P_{(V_{j_1})^{c_1}}} (V_{j_1})^{c_1} \xrightarrow[\text{conv.}]{\Psi_2} (L^2(\mathbb{R}))^{c_2} \\  \xrightarrow[\text{proj.}]{P_{(V_{j_2})^{c_2}}} (V_{j_2})^{c_2} 
\xrightarrow[\text{nonlin.}]{\sigma_2} (L^2(\mathbb{R}))^{c_2} \xrightarrow[\text{proj.}]{P_{(V_{j_2})^{c_2}}} (V_{j_2})^{c_2} \xrightarrow[\text{conv.}]{\Psi_3} \cdots \\  \cdots \xrightarrow[\text{conv.}]{\Psi_{L}} L^2(\mathbb{R}) \xrightarrow[\text{proj.}]{P_{V_{j_L}}} V_{j_L}
\xrightarrow[\text{nonlin.}]{\sigma_{L}} L^2(\mathbb{R}) \xrightarrow[\text{proj.}]{P_{V_{j_L}}} V_{j_L},
\end{multline*}
which can be summarized as
\begin{equation}
\label{continuous_generator}
    G= \left( \comp_{l=L}^2 \tilde{\sigma}_{l} \circ \tilde\Psi_l\right)
    \circ \left(\tilde{\sigma}_1 \circ \Psi_1\right),
\end{equation}
where
\begin{equation}
\label{tilde_psi}
\tilde{\Psi}_{l} := P_{(V_{j_{l}})^{c_{l}}} \circ \Psi_{l}\colon (V_{j_{l-1}})^{c_{l-1}}\to (V_{j_{l}})^{c_{l}}, \qquad l=2,...,L,   
\end{equation}
and
\begin{equation}
\label{tilde_sigma}
\tilde{\sigma}_{l} := P_{(V_{j_{l}})^{c_{l}}} \circ \sigma_l \colon (V_{j_{l}})^{c_{l}} \to (V_{j_{l}})^{c_{l}}, \qquad l=1,...,L.  
\end{equation}
\begin{remark}[On the finite-dimensionality of the network $G$]
Even though the scale $j$ is fixed,  the spaces $V_j$ in \eqref{tilde_psi} and \eqref{tilde_sigma} are infinite-dimensional because of the infinite translations in \eqref{phi_j_n}. However, when restricting to functions with a fixed compact support (as is done in practice, and as we will do below), every layer contains maps between finite-dimensional spaces. On the one hand, this will allow for a relatively simple implementation of the network, similarly to a discrete neural network (see $\S$\ref{implementation} below). On the other hand, this architecture avoids any further arbitrary (e.g.\ pixel-based) discretization, yielding better results for analog signals. This aspect is related to the discretization issue in operator learning, see \cite{bartolucci2023representation}.
\end{remark}

\begin{remark}[Idea behind the continuous strided convolution]
Let us focus on the case with stride $s=2$ for simplicity. The cases with $s=2^\nu$ with $\nu\ge 2$ are analogous, while the corresponding deconvolutions ($s=2^\nu$ with $\nu\le -1$) are simply obtained by taking the adjoint operator, as the discrete deconvolution is obtained by taking the transpose of the convolution. The discrete convolution with stride $s = 2$ (see equation \eqref{conv_discr} in \ref{proof_adjoint} for more details) is
obtained by
\begin{enumerate}
    \item taking a standard discrete convolution;
    \item and by keeping only every second entry of the resulting vector.
\end{enumerate}
As a consequence, the resolution of the output vector is half  that of the input vector (ignoring boundary effects).

Our definition of the continuous strided convolution \eqref{cont_strid_conv} generalizes these operations. If we start with an input signal $f \in V_{j+1}$ and a filter $g\in V_{j+1}$, the resulting convolution is $P_{V_j}(f*g)$, namely we
\begin{enumerate}[label=(\roman*)]
    \item take a continuous convolution $f * g$;
    \item and project the resulting signal onto $V_{j}$.
\end{enumerate}
Here, (i) is the natural continuous version of 1. In (ii), the projection onto $V_j$ consists of local averages, which correspond to step $2$., where, instead of taking averages, only every second entry of the output vector was kept. Further, the input vector belongs to $V_{j+1}$ and the output vector to $V_j$, and so the resolution of the latter is half  that of the former, as in the discrete case. Finally, in order to define the convolution on the whole space $L^2(\R)$, the input vector is first projected onto $V_{j+1}$.
\end{remark}

%\begin{remark}[Relation between discrete and continuous strided convolutions]\label{rem:disc_cont}

The following lemma is useful to understand the relation between discrete and continuous strided convolutions.

\begin{lem}
\label{formula_compressa_conv_proj}
Given a function $f =\sum_{n \in \mathbb{Z}} c_n \phi_{j,n}\in V_j$, then 
\begin{equation*}
 f *_{j\to j+\nu} g = \sum_{m \in \Z} \langle  f \ast g, \phi_{j+\nu,m} \rangle_2 \phi_{j+\nu,m},  
\end{equation*}
where $g = \sum_{p \in \mathbb{Z}} d_p \phi_{j+\nu,p}\in V_{j+\nu} \cap L^1(\R)$, represents a filter, and 
\begin{equation}\label{conv_proj}
  \langle  f \ast g, \phi_{j+\nu,m} \rangle_2  = 2^{-\frac{j}{2}} \sum_{n \in \mathbb{Z}} c_n d_\eta (m-2^{\nu}n) = 2^{-\frac{j}{2}} \big(c \ast_{\frac{1}{2^\nu}} d_\eta \big)(m),
\end{equation}
where $\ast_s$ is the discrete strided convolution with stride $s$ defined in \eqref{deconv_discr}, $d_\eta := d \ast_1 \eta_\nu$ and $\eta_\nu$ is the sequence defined as
\begin{equation}
\label{eta}
  \eta_\nu(r) := \int_{\R^2} \phi(t) \phi(z) \phi(z-2^\nu t+r) dz dt, \qquad r \in \Z.
\end{equation}
\end{lem}
\begin{proof}
We first prove that $\langle \phi_{j,n} \ast \phi_{j+\nu,p}, \phi_{j+\nu,m} \rangle_2$ depends only on the scaling function $\phi$, the scale $j$ and the coefficient $m-p-2^\nu n$. We have
\begin{align}\notag
\langle \phi_{j,n} \ast \phi_{j+\nu,p} , \phi_{j+\nu,m} \rangle_2 & = \int_{\mathbb{R}^2} \phi_{j,n}(y) \phi_{j+\nu,p}(x-y) \phi_{j+\nu,m}(x) dx dy \\ &  = \int_{\mathbb{R}^2} 2^{\frac{j}{2}} \phi\Big(2^{j} y - n\Big) 2^{\frac{j+\nu}{2}} \phi(2^{j+\nu}(x-y)-p) 2^{\frac{j+\nu}{2}} \phi(2^{j+\nu} x-m) dx dy  \notag \\ &= 2^{-\frac{j}{2}} \int_{\mathbb{R}^2} \phi(t) \phi(z) \phi(z-2^\nu t+m-2^\nu n-p) dz dt  \notag \\ & = 2^{-\frac{j}{2}} \hspace{0.05cm} \eta_\nu(m-p-2^\nu n),
\label{eta_compatto}
\end{align}
where $\phi_{j,n} = 2^{\frac{j}{2}} \phi(2^{j} \cdot - n) $ and $\eta_\nu(r) = \int_{\R^2} \phi(t) \phi(z) \phi(z-2^\nu t+r) dz dt$, as defined in \eqref{phi_j_n} and \eqref{eta}, respectively. We obtain
\begin{equation*}
\begin{split}
  \langle f \ast g, \phi_{j+\nu,m} \rangle_2 & = \sum_{n \in \mathbb{Z}} \sum_{p \in \mathbb{Z}} c_n d_p \langle \phi_{j,n} \ast \phi_{j+\nu,p}, \phi_{j+\nu,m} \rangle_2 \\&
  = 2^{-\frac{j}{2}} \sum_{n \in \mathbb{Z}} c_n \sum_{p \in \mathbb{Z}} d_p \eta(m-p-2^\nu n) \\ &
  = 2^{-\frac{j}{2}} \sum_{n \in \mathbb{Z}} c_n d_\eta (m-2^{\nu}n) \\ & = 2^{-\frac{j}{2}} \big(c \ast_{\frac{1}{2^\nu}} d_\eta \big)(m),
\end{split}
\end{equation*}
where $\ast_s$ is the discrete strided convolution with stride $s$ defined in \eqref{deconv_discr}, $d_\eta := d \ast_1 \eta_\nu$ and $\eta_\nu(r)$ is the sequence defined in \eqref{eta}.
\end{proof}

%If we apply the continuous strided convolution \eqref{j_to_j+nu} to a function $f =\sum_{n \in \mathbb{Z}} c_n \phi_{j,n}\in V_j$, we obtain 
%\begin{equation*}
 %f *_{j\to j+\nu} g = \sum_{m \in \Z} \langle  f \ast g, \phi_{j+\nu,m} \rangle_2 \phi_{j+\nu,m},  
%\end{equation*}
%where $g = \sum_{p \in \mathbb{Z}} d_p \phi_{j+\nu,p}\in V_{j+\nu} \cap L^1(\R)$ represents a filter. It can be shown (Lemma~\ref{formula_compressa_conv_proj} in \ref{ap:cont_strid_conv}) that
%\begin{equation}
%\label{conv_proj}
%\begin{split}
 % \langle  f \ast g, \phi_{j+\nu,m} \rangle_2 = 2^{-\frac{j}{2}} \big(c \ast_{\frac{1}{2^\nu}} d_\eta \big)(m),
%\end{split}
%\end{equation}
%where $\ast_s$ is the discrete strided convolution with stride $s$ defined in \eqref{deconv_discr}, $d_\eta := d \ast_1 \eta_\nu$ and $\eta_\nu$ is the sequence defined as
%\begin{equation}
%\label{eta}
%  \eta_\nu(r) := \int_{\R^2} \phi(t) \phi(z) \phi(z-2^\nu t+r) dz dt, \qquad r \in \Z.
%\end{equation}
Therefore, the coefficients of the output signal (with respect to $\{\phi_{j+\nu,m}\}$) are obtained by taking a discrete strided convolution of the coefficients $c$ of the input signal (with respect to $\{\phi_{j,m}\}$)  with the (discrete) filter $d_\eta$, which is obtained by taking a discrete convolution of the coefficients of the filter $g$ with $\eta_\nu$.
In other words, a continuous strided convolution between two signals can be seen as a discrete strided convolution between the corresponding coefficients in the natural basis. 
%\end{remark}

\paragraph{A simple example: the Haar case}
Let $V_j$ be the spaces of the MRA associated to the Haar scaling function $\phi = \mathbbm{1}_{[0,1]}$. This simple setting naturally extends the discrete case to the continuous one. Indeed, given a signal $f\in V_j$, namely a piecewise constant function on dyadic intervals, its coefficients with respect to the family $\{\phi_{j,m}\}$ are the values of $f$ itself (up to a normalization factor), and so discrete and continuous convolutions almost coincide (see \eqref{conv_proj}). The only difference being in  the filter, since in this case $\eta_1 = [...,0,0.25,0.5,0.25,0,...] \in \R^{\mathbb{Z}}$, and so $d\neq d_\eta$. 
We would have exact correspondence  if $\eta$ were the Dirac delta, i.e.\ $\eta  = [...,0,0,1,0,0,...] \in \R^{\mathbb{Z}}$ (but this cannot be obtained with any choice of wavelet), so that $d_\eta=d$. Moreover, the Haar scaling function makes it possible to simplify the structure of a CGNN. Indeed $\sigma_{l} (V_{j_{l}})\subset V_{j_{l}}$, thanks to the form of $\phi$ and the fact that $\phi_{j_{l},k_1}$ and $\phi_{j_{l},k_2}$ have disjoint support for every $k_1 \neq k_2$. Therefore, in this setting, the projections after the nonlinearities can be removed.

\subsection{Details on the implementation of the network}
\label{implementation}

In order to implement our generator \eqref{continuous_generator}, which is written as a composition of maps between infinite-dimensional spaces, $\tilde\Psi_l$ and $\tilde\sigma_l$, we need to find a proper discretization of the network, ideally avoiding fine discretizations of the space. 

\paragraph{Implementation of $\tilde\Psi_l$} 
The continuous strided convolution $\tilde\Psi_l$, namely a continuous convolution followed by the projection, may be efficiently computed by using \eqref{conv_proj},  if we set the bias term to be zero for each output channel.
Indeed, we require only one computation of a continuous integral (i.e.\ the integral that defines $\eta_\nu$ in \eqref{eta}, which depends only on the choice of the wavelet) and a series of discrete convolutions. In the case when the bias term is not trivial, this can be dealt with directly at the level of the wavelet coefficients.

\paragraph{Implementation of $\tilde\sigma_l$} 
The computation of $\tilde\sigma_l$, i.e.\ the nonlinearity followed by the projection, is more subtle (apart in the case of the Haar wavelet, where it is enough to apply the nonlinearity to the scaling coefficients).
The most straightforward implementation would be to consider a fine discretization of the space, which would allow for the computation of the pointwise nonlinearity and of the integrals involved in the scalar products related to the projection. However, this  would be computationally heavy.
Instead, we propose to approximate the pointwise values of a signal belonging to a certain $V_j$ by using its scaling coefficients in $V_{j+M}$ with $M$ sufficiently large. We apply the nonlinearity directly to these coefficients, and then we project back to $V_j$. Note that going from $V_j$ to $V_{j+M}$ and back can be  performed  very efficiently by using the fast wavelet transform, and the operations of upsampling and downsampling (see \cite[Section $5.6$]{Daubechies}).

We now clarify how the scaling coefficients in $V_{j+M}$ for large $M$ provide a (pointwise) discretization of the original signal in $V_j$. We have 
\begin{equation}
\label{scal_coeff_e_funz}
    \lim_{j \to +\infty} \frac{2^{\frac{j}{2}} \langle f, \phi_{j,2^{j} b} \rangle_2}{\int_{\R} \phi\,dx} = f(b), \qquad \text{a.e. } b \in \R, 
\end{equation} 
whenever $f \in L^2(\R)$ and $\phi$ is compactly supported, bounded and $\int_{\R} \phi\,dx \neq 0$. If $f$ is continuous, equation~\eqref{scal_coeff_e_funz} holds for every $b \in \R$.
In other words, the scaling coefficients at fine scales approximate the pointwise values of the function, up to a constant. 
In the Haar case, equation~\eqref{scal_coeff_e_funz} is a simple consequence of the Lebesgue differentiation theorem, while in the general case, it follows from an extension (see \cite[Corollary 2.1.19]{grafakos2008classical}).

\paragraph{Summing up} 
Altogether, the implementation of the continuous network may be written as
\begin{multline*}
G \colon \mathbb{R}^{\Esse}  \xrightarrow[\text{f.c.}]{\Psi_1^\mathrm{w}} (\mathbb{R}^{\alpha_1})^{c_1} \xrightarrow[\text{nonlin.}]{\sigma_1^\mathrm{w}} (\mathbb{R}^{\alpha_1})^{c_1} \xrightarrow[\text{conv.}]{\Psi_2^\mathrm{w}} (\mathbb{R}^{\alpha_2})^{c_2} \xrightarrow[\text{nonlin.}]{\sigma_2^\mathrm{w}} (\mathbb{R}^{\alpha_2})^{c_2} \xrightarrow[\text{conv.}]{\Psi_3^\mathrm{w}}  \cdots \\ 
\cdots \xrightarrow[\text{conv.}]{\Psi_{L}^\mathrm{w}}  \mathbb{R}^{\alpha_{L}} \xrightarrow[\text{nonlin.}]{\sigma_{L}^\mathrm{w}} \mathbb{R}^{\alpha_{L}}\xrightarrow[\text{synth.}]{S_L} V_{j_L},
\end{multline*}
where:
\begin{itemize}
    \item the map $\Psi_1^\mathrm{w}$ represents each component of $\Psi_1$ with respect to an orthonormal system  $\{\phi_{j_1,n}\}_n$ of $V_{j_1}$ (we restrict to a fixed compact support, so that only $\alpha_1$ indices $n$ are considered);
    \item as for $\Psi_1^\mathrm{w}$, for $l=2,\dots,L$, the map $\Psi_l^\mathrm{w}$ is the map $\tilde\Psi_l\colon (V_{j_{l-1}})^{c_{l-1}}\to (V_{j_{l}})^{c_{l}}$ written with respect to suitable orthonormal systems $\{\phi_{j,n}\}_n$ of $V_{j}$ and, as discussed above, is nothing but a discrete convolution with a particular filter;
    \item the nonlinearity $\sigma_l^\mathrm{w}\colon (\mathbb{R}^{\alpha_l})^{c_l}\to (\mathbb{R}^{\alpha_l})^{c_l} $, written between scaling coefficients, is the composition of an upsampling, a pointwise nonlinearity, and a  downsampling (below, we will also consider a simplified version without upsampling and downsampling, see Remark~\ref{simply_cgnn});
    \item and, finally, the synthesis operator $S_L$ outputs a function in $V_{j_L}$ given its scaling coefficients.
\end{itemize}
As such, the implementation of the CGNN closely resembles a discrete GNN ($\S$\ref{sub:1Ddiscrete}). The key differences lie in the presence of the final synthesis operator, in the nonstandard nonlinearities, and in the double convolution in \eqref{conv_proj} with a fixed filter $\eta_\nu$.

\section{Injectivity of CGNNs}
\label{sec:injectivity}

We are interested in studying the injectivity of the continuous generator \eqref{continuous_generator} to guarantee uniqueness in the representation of the signals. The injectivity will also allow us, as a by-product, to obtain stability results for inverse problems using generative models, as in Section~\ref{sec:deep_land}. 

We consider here the 1D case with stride $s = \frac{1}{2}$ ($\nu = 1$); then, applying \eqref{stride_e_j} iteratively, we obtain $j_{l} = j_1 + l -1$ for $l=1,\dots,L$. We also consider non-expansive convolutional layers, i.e.\ $c_{l} = \frac{c_{l-1}}{2} = s c_{l-1} = \frac{c_1}{2^{l-1}}$. We note that the same result holds also with expansive convolutional layers, arbitrary stride (possibly dependent on $l$) and in the 2D case (see Appendix~\ref{sec:extension}).

We make the following assumptions.

\paragraph{Assumptions on the scaling spaces $V_{j_{l}}$}
\begin{hyp}
\label{hyp_scaling_spaces}
The spaces $V_{j_{l}}\subset L^2(\R)$, with $j_{l} \in \N$, belong to an MRA (see Definition~\ref{MRA}), whose scaling function $\phi$ is compactly supported and bounded. Furthermore, there exists $r \in \mathbb{Z}$ such that
    \begin{equation}
    \label{eta_diverso_0}
    \eta_1(r) \neq 0,  
    \end{equation}
  where $\eta_1(r)$ is defined in \eqref{eta}.
\end{hyp}

    We observe that Hypothesis~\ref{hyp_scaling_spaces} implies that the sequence $\eta_1$ defined in \eqref{eta} has a finite number of non-zero elements. As a consequence, when $g$ is compactly supported, the filters $d_\eta$ can be represented by finite-dimensional vectors, as for discrete neural networks. 
    However, as highlighted in Lemma~\ref{formula_compressa_conv_proj}, the continuous strided convolution involves a double convolution \eqref{conv_proj}, which is not present in the discrete strided convolution.

\begin{remark}[Haar and Daubechies scaling functions]
\label{matlab_code_eta}
For positive functions, such as the Haar scaling function, i.e.\ $\phi = \mathbbm{1}_{[0,1]}$, condition \eqref{eta_diverso_0} is easily satisfied. For the Daubechies scaling functions with $N$ vanishing moments for $N = 1,2,...,45$ ($N = 1$ corresponds to the Haar scaling function), we verified  condition~\eqref{eta_diverso_0} numerically. We believe that this condition is satisfied for every scaling function $\phi$, but have not been able to prove this rigorously.

We note, however, that Hypothesis~\ref{hyp_scaling_spaces} is clearly needed for the injectivity of the convolutional layers. Indeed if we had $\eta_1(r)=0$ for every $r\in\Z$, then
\begin{equation*}
    \langle \phi_{j,n} * \phi_{j+1,p}\,,\,\phi_{j+1,m}\rangle_2=0,\qquad j,n,p,m\in\Z.
\end{equation*}
Therefore, Hypothesis~\ref{hyp_scaling_spaces} guarantees that the continuous deconvolution $\cdot *_{j\to j+1} g$ is not identically $0$ (at least for some filter $g\in V_{j+1}$).
\end{remark}

\paragraph{Assumptions on the convolutional filters}
The following hypothesis asks that, at each convolutional layer, the filters are compactly supported, where $\bar{p} + 1$ represents the filters' size. Generally, the convolutional filters act locally, so it is natural to assume that they have compact support. Furthermore, we ask the filters to be linearly independent, in a suitable sense: this is needed for the injectivity of the convolutional layers.

\begin{hyp}
\label{hyp_filtri}
Let $\bar p\in\N$. For every $l=2,...,L$, the convolutional filters $t^{l}_{i,k} \in V_{j_{l}}$ of the $l$-th convolutional layer \eqref{psi_ell} satisfy  
    \begin{equation}
    \label{decomposition_filter}
    t^{l}_{i,k} = \sum_{p=0}^{ \bar{p}} d^{l}_{p,i,k} \phi_{j_{l},p}, \qquad i=1,...,c_{l-1}, \qquad k=1,...,c_{l},
    \end{equation}
    where $d^{l}_{p,i,k} \in \R$, and
 $\det({D}^{l}) \neq 0$, where ${D}^{l}$ is the $\frac{c_1}{2^{l-1}} \times \frac{c_1}{2^{l-1}}$ matrix defined by
\begin{equation}
\label{D_0}
({D}^{l})_{i,k}  :=
    \begin{cases}
      d^{l}_{0,i,k} & k = 1,...,\frac{c_1}{2^{l}},\\
      d^{l}_{1,i,k-\frac{c_1}{2^{l}}} & k = \frac{c_1}{2^{l}}+1,...,\frac{c_1}{2^{l-1}}.
    \end{cases}
\end{equation}
\end{hyp}

\begin{remark}
The condition $\det({D}^{l}) \neq 0$ is sufficient for the injectivity of the convolutional layers, but not necessary. The necessary condition is given in \ref{injectivity_proof_supp_mat}, and consists in requiring the rank of a certain block matrix to be maximum, in which $D^l$ is simply the first block. We note that this condition is  independent of the scaling function $\phi$, but depends only on the filters' coefficients $d^{l}_{p,i,k}$.
\end{remark}

\begin{remark}[Analogy between continuous and discrete case]
\label{analogy_cont_discr}
    The splitting operation of the filters' scaling coefficients in odd and even entries, as in Hypothesis~\ref{hyp_filtri}, reminds the expression of the discrete convolution with stride $s = \frac{1}{2}$. Indeed, a discrete filter $t$ is split into $t_0$, containing the even entries of $t$, and $t_1$, containing the odd ones (see \ref{proof_adjoint})
\end{remark}

\paragraph{Assumptions on the nonlinearity}
For simplicity, we consider the same nonlinearity $\sigma \colon \mathbb{R} \to \mathbb{R}$ in each layer (the generalization to the general case is straightforward). The following conditions guarantee that $\tilde\sigma_l$ is injective.
\begin{hyp} 
\label{hyp_sigma}
We assume that 
\begin{enumerate}
    \item $\sigma \in C^1(\mathbb{R})$ is injective and $M_1\le \sigma'(x) \le M_2$ for every $x \in \R$, for some $M_1,M_2>0$; \label{cond_1}
    \item $\sigma(0)=0$ and $\sigma$ preserves the sign, i.e.\ $x \cdot \sigma(x) \geq 0$ for every $x \in \mathbb{R}$; \label{cond_2}
\end{enumerate}    
\end{hyp}

Note that these conditions  ensure that $|\sigma(x)| \leq L |x|$ for every $x \in \R$, for some $L>0$, and so $\sigma_{l}(f) \in (L^2(\R))^{c_{l}}$ for every $f \in (L^2(\R))^{c_{l}}$. It is also straightforward to check that the injectivity of $\sigma \colon \R \to \R$ ensures the injectivity of 
    \[
    \sigma_{l} \colon (V_{j_{l}})^{c_{l}}  \to (L^2(\mathbb{R}))^{c_{l}},\qquad 
    f \mapsto \sigma_{l}(f).
    \]

\begin{remark}
\label{remark_haar_sigma}
In the Haar case, the projection after the nonlinearity can be removed, as explained at the end of Section~\ref{sec:1D_continuous_generator_architecture}, and we need to  verify only the injectivity of $\sigma_{l}$ instead of that of $\tilde\sigma_l=P_{(V_{j_{l}})^{c_{l}}} \circ \sigma_{l}$. As noted above, a sufficient condition to guarantee the injectivity of $\sigma_{l}$ is the injectivity of $\sigma$. So, in the Haar case, Hypothesis~\ref{hyp_sigma} can be relaxed and replaced by:
\vspace{-2mm}
\begin{enumerate}
    \item $\sigma$ is injective;
    \item There exists $L >0$ such that $|\sigma(x)| \leq L |x|$ for every $x \in \R$.
\end{enumerate}  
\end{remark}

Hypothesis~\ref{hyp_sigma} is satisfied for example by the function $\sigma_{\rm hp}(x) = |x| \arctan(x)$. Its relaxed version, in the Haar case,  is satisfied by some commonly used nonlinearities, such the Sigmoid, the Hyperbolic tangent, the Softplus, the Exponential linear unit (ELU) and the Leaky rectified linear unit (Leaky ReLU). Our approach does not allow us to consider non-injective $\sigma$'s, such as the ReLU \cite{injective_relu}.

For simplicity, in Hypothesis~\ref{hyp_sigma}, we require that $\sigma \in C^1(\R)$ and that $\sigma'$ is strictly positive everywhere. This allows us to use Hadamard's global inverse function theorem \cite{Gordon_diffeomorfismo} to obtain the injectivity of the generator. However, thanks to a generalized version of Hadamard's theorem \cite{Pourciau}, we expect to be able to relax the conditions by requiring only that $\sigma$ is Lipschitz and its generalized derivative is strictly positive everywhere. In this way, the Leaky ReLU would satisfy the assumptions.

\paragraph{Assumptions on the fully connected layer}
We impose the following natural hypothesis on the fully connected layer. 

\begin{hyp}
\label{hyp_f_c}
We assume that 
\begin{enumerate}[leftmargin=0.7cm]
    \item The linear function $F \colon \R^{\Esse} \to (V_{j_1})^{c_1}$ is injective;
    \item There exists $N \in \N$ such that $b \in (\spann\{ \phi_{j_1,n} \}_{n=-N}^{N})^{c_1}$ and $\Im(F) \subset (\spann\{ \phi_{j_1,n} \}_{n=-N}^{N})^{c_1}$. 
\end{enumerate}
\end{hyp}
The inclusion $F(\R^{\Esse})+b \subset (V_{j_1})^{c_1}$ is natural, since
we start with low-resolution signals. The second condition in Hypothesis~\ref{hyp_f_c} means that the image of the first layer, $F(\R^{\Esse})+b$, contains only compactly supported functions. This is natural since we deal with signals of finite size. 

Even the injectivity of $F$ is non-restrictive, since we choose the dimension $\Esse$ of the latent space to be much smaller than the dimension of $\Im(F) = (\spann\{ \phi_{j_1,n} \}_{n=-N}^{N})^{c_1}$, which is $c_1(2N+1)$. So, $F$ maps a low-dimensional space into a higher dimensional one.

\paragraph{The injectivity theorem}
The main result of this section reads as follows.
\begin{thm}
\label{main_thm}
Let $L \in \mathbb{N}^*$ and $j_1 \in \mathbb{Z}$. Let $c_1 = 2^{L-1}$, $c_{l} = \frac{c_1}{2^{l - 1}}$ and $j_{l} = j_1 + l - 1$ for every $l=1,...,L$. Let $V_{j_{l}}$ be the scaling function space arising from an MRA, and $t^{l}_{i,k} \in V_{j_{l}}$ for every $l=2,...,L$, $i=1,...,c_{l-1}$ and $k=1,...,c_{l}$. Let $\tilde{\Psi}_{l}$ and $\tilde{\sigma}_{l}$ be defined as in \eqref{tilde_psi} and \eqref{tilde_sigma}, respectively. Let $\Psi_1$ be defined as in \eqref{eq:Psi1}.
If Hypotheses~\ref{hyp_scaling_spaces}, \ref{hyp_filtri}, \ref{hyp_sigma} and \ref{hyp_f_c} are satisfied, then the generator $G$ defined in \eqref{continuous_generator} is injective.
\end{thm}

\begin{proof}[Sketch of the proof]
Consider \eqref{continuous_generator}, \eqref{tilde_psi} and \eqref{tilde_sigma}. Note that $\Psi_1$ is injective by Hypothesis~\ref{hyp_f_c}. If we also show that $\tilde\Psi_l$ is injective for every $l=2,\dots,L$ and that $\tilde{\sigma}_{l}$ is injective for every $l=1,\dots,L$, then the injectivity of $G$ will immediately follow.

The injectivity of $\tilde{\Psi}_{l}$ is a consequence of Hypothesis~\ref{hyp_filtri} (together with Hypothesis~\ref{hyp_scaling_spaces}). The injectivity of $\tilde{\sigma}_{l}$ follows from Hypothesis~\ref{hyp_sigma} (together with Hypotheses~\ref{hyp_scaling_spaces} and \ref{hyp_f_c}) and from Hadamard's global inverse function theorem applied to $\tilde{\sigma}_{l}$. The full proof is presented in \ref{injectivity_proof_supp_mat}.
\end{proof}

\begin{remark}[Simplified CGNN architecture]
\label{simply_cgnn}
It is possible to consider a simplified CGNN architecture in which the nonlinearities $\sigma$ are applied on the signal scaling coefficients, i.e.\ $\sigma(f) =\sum_{n \in \mathbb{Z}} \sigma(c_n) \phi_{j,n}$ with $f =\sum_{n \in \mathbb{Z}} c_n \phi_{j,n}\in V_j$. In this case, the projections onto the scaling spaces following the nonlinearities are not needed because $\sigma(f) \in V_j$. Therefore, the injectivity of $\tilde{\sigma}_{\ell}$ is guaranteed by assuming only the injectivity of $\sigma$. Then, the injectivity of the simplified CGNN follows from Theorem~\ref{main_thm} by replacing Hypothesis~\ref{hyp_sigma} with the injectivity of $\sigma$.
\end{remark}

\section{Stability of inverse problems with generative models}
\label{sec:deep_land}
We now show how an injective CGNN can be used to solve ill-posed inverse problems. The purpose of a CGNN is to reduce the dimensionality of the unknown to be determined, and the injectivity is the main ingredient to obtain a rigorous stability estimate.

We consider an inverse problem of the form
\begin{equation}
\label{IP}
    y = \mathcal{F}(x),
\end{equation}
where $\mathcal{F}$ is a possibly nonlinear map between Banach spaces, $X$ and $Y$, modeling a measurement (forward) operator, $x \in X$ is a quantity to be recovered and $y \in Y$ is the noisy data. Typical inverse problems are ill-posed (e.g.\ CT, accelerated MRI, or electrical impedance tomography), meaning that the noise in the measurements is amplified in the reconstruction. For instance, in the linear case, this instability corresponds to having an unbounded (namely, not Lipschitz) inverse $\mathcal F^{-1}$. The ill-posedness is classically tackled by using regularization, which often leads to an iterative method, as the gradient-type Landweber algorithm \cite{Engl}. This can be very expensive if $X$ has a large dimension. 

However, in most of the inverse problems of interest, the unknown $x$ can be modeled as an element of a low-dimensional manifold in $X$. We choose to use a generator $G \colon \R^{\Esse} \to X$ to perform this dimensionality reduction and therefore our problem reduces to finding  $z \in \R^{\Esse}$ such that
\begin{equation}
\label{IPG}
    y = \mathcal{F}(G(z)).
\end{equation}
In practice, the map $G$ is found via an unsupervised training procedure, starting from a training dataset.
From the computational point of view, solving \eqref{IPG} with an iterative method is clearly more advantageous than solving \eqref{IP}, because $z$ belongs to a lower dimensional space. We note that the idea of solving inverse problems using deep generative models has been considered in \cite{IPwithDL,Seo, AKR19, Seo1, Seo2, Dimakis, ulyanov2018deep,asim2020blind}.

The dimensionality reduction given by the composition of the forward operator with a generator as in \eqref{IPG}, has a regularizing/stabilizing effect that we aim to quantify. More precisely, we show that an injective CGNN yields a Lipschitz stability result for the inverse problem \eqref{IPG}; in other words, the inverse map is Lipschitz continuous, and noise in the data is not amplified in the reconstruction. For simplicity, we consider the $1$D case with stride $\frac 1 2$ and non-expansive convolutional layers, but the result can be extended to the $2$D case and arbitrary stride as done in \ref{sec:extension} for Theorem~\ref{main_thm}.

\begin{thm} \label{lipschitz_stab}
Let $X = L^2(\R)$ and $G$ be a CGNN satisfying Hypotheses~\ref{hyp_scaling_spaces}, \ref{hyp_filtri}, \ref{hyp_sigma} and \ref{hyp_f_c}. Let $\mathcal{M} := G(\R^{\Esse})$,  $K \subset \mathcal{M}$ be a compact set, $Y$ be a Banach space and $\mathcal F\colon X\to Y$ be a $C^1$ map (possibly nonlinear).
Assume that $\mathcal{F}$ is injective and  $\mathcal F'(x)|_{T_x \mathcal M}$ is injective for every $x \in \mathcal{M}$. Then there exists a constant $C > 0$ such that
\begin{equation*}
    \| x-y \|_{X} \leq C \|\mathcal{F}(x)-\mathcal{F}(y)\|_{Y}, \qquad x, y \in K.
\end{equation*}
\end{thm}

The proof of Theorem~\ref{lipschitz_stab} can be found in \ref{app:lip} and is mostly based on Theorem~\ref{main_thm} and \cite[Theorem 2.2]{Alberti_Santacesaria_Arroyo}. This Lipschitz estimate can also be obtained in the case when only finite measurements are available, i.e.\ a suitable finite-dimensional approximation of $\mathcal{F}(x)$, thanks to \cite[Theorem $2.5$]{Alberti_Santacesaria_Arroyo}. A similar estimate can be derived for $\|G^{-1}(x) - G^{-1}(y)\|_{\R^S}$ (see \cite[Proof of Theorem~2.2]{Alberti_Santacesaria_Arroyo}), even though this bound in the latent space is less relevant for inverse problems.

The Lipschitz stability estimate provided in Theorem~\ref{lipschitz_stab} ensures the convergence of the Landweber algorithm (see \cite{de2012local}), which can be used as a reconstruction method.
In our setting, this  algorithm is applied to the functional $\mathcal{F} \circ G$ and, given an initial guess $z_0$, it produces a sequence of iterations $z_k$ 
\begin{equation}
\label{deep_landweber}
    z_k = z_{k-1} - h \nabla(\mathcal{F} \circ G)(z_{k-1}), \qquad k \geq 1, 
\end{equation}
where $h > 0$ is the stepsize.

Although the injectivity of the generator is not a necessary condition for solving inverse problems of the type \eqref{IPG}, in order to prove theoretical results about inverse problems, it is still a mandatory assumption, because of the proof techniques. It is possible nevertheless that global injectivity of the generator is not needed, and a weaker local one might suffice. This could be justified by using a differential geometric approach, as in \cite{Alberti_Santacesaria_Arroyo}. In Remark~\ref{toy_ex_inj} below, we provide a toy example in which a non-injective generator makes the Landweber iteration not convergent. However, it is still not clear, from the theoretical point of view, whether non-injective models can be successfully used to solve inverse problems.

\begin{remark}
\label{toy_ex_inj}
We demonstrate how a non-injective generative model can lead to difficulties in solving inverse problems. For simplicity, we consider a simple one-dimensional toy example, with a ReLU activation function $\sigma$, which is not injective.
Let
\[
G\colon\R\to\R,\qquad G(z) = \sigma(-\sigma(z)+1)=(-z_+ +1)_+,
\]
where $z_+ = \max(z,0)$. This is a simple 2-layer neural network, where the affine map in the first layer is $z\mapsto z$, and the affine map of the second layer is $x\mapsto -x+1$. It is immediate to see that this network generates the set $[0,1]\subset\R$. However, it is not injective:
\begin{equation}
\label{eq:Gnotinj}
G(z)=
\begin{cases}
  1 &\text{if $z\le 0$,}\\
  0 &\text{if $z\ge 1$.}
\end{cases}
\end{equation}

Suppose now we wish to solve the inverse problem $\mathcal F(x)=y$, as in \eqref{IP} with any map $\mathcal F$, for $x\in [0,1]$. Writing $x=G(z)$, we are reduced to solving \eqref{IPG}, namely $\mathcal F(G(z))=y$. If we solve this by using any iterative method, as \eqref{deep_landweber}, with an initial guess $z_0\in\R\setminus [0,1]$, we have $z_k=z_0$ for every $k$, since $G$ is constant in a neighborhood of $z_0$ by \eqref{eq:Gnotinj}. Therefore, in general, it will not be possible to solve the inverse problem with a non-injective generative model.

This issue appears whenever the support of the latent distribution for $z$ is larger than $[0,1]$, where $G$ is injective. It is worth noting that this issue would be solved if we restrict $G$ to  $[0,1]$, the effective support of the ideal latent distribution. This aspect is connected to the behavior of generative models with out-of-distribution data, and a detailed discussion goes beyond the scopes of this work.
\end{remark}

\section{Numerical results}\label{sec:num}
We present here numerical results validating our theoretical findings. In Section~\ref{sec:training} we describe how we train a CGNN, in Section~\ref{sec:deconv} we apply a CGNN-based reconstruction algorithm to signal deblurring, and in Section~\ref{appendix:generation} we show qualitative results on the generation and reconstruction capabilities of CGNNs.

\subsection{Training}
\label{sec:training}
The conditions for injectivity given in Theorem~\ref{main_thm} are not very restrictive and we can use an unsupervised training protocol to choose the parameters of a CGNN. Even though our theoretical results concern only the injectivity of CGNNs, we numerically verified that training a generator to also well-approximate a probability distribution gives better reconstructions for inverse problems. For this reason, we choose to train CGNNs as parts of variational autoencoders (VAEs) \cite{kingma2013auto}, a popular architecture for generative modeling. In particular, our VAEs are designed so that the corresponding decoder has a CGNN architecture. We refer to \cite{VAE} for a thorough review of VAEs. Note that there is growing numerical evidence showing that an untrained convolutional network is competitive with trained ones for solving inverse problems with generative priors \cite{ulyanov2018deep,asim2020blind}. In these cases, the network is initialized at random, and the weights are obtained by minimizing the fidelity term corresponding to a given inverse problem. As a result, the architecture of the network itself acts as a regularizer.

In our experiment, the training is done on smooth signals. We create a dataset of smooth signals constructed by randomly sampling the first $5$ low-frequency Fourier coefficients. Each of these is taken from a Gaussian distribution with zero mean and variance that decreases as the frequency increases. The decay of the Fourier coefficients is used to enforce smoothness. In order to show the validity of our method in a high resolution setting, the support of the signals $[0,1]$ is finely discretized with $4096$ equidistant points. We divide the dataset in $10000$ signals to train the network and $2000$ to test it.

Our VAE consists of a decoder with 3 non-expansive transposed convolutional layers without bias terms and an encoder with a similar, yet mirrored, structure. For the decoder, we consider the simplified CGNN architecture described in Remark~\ref{simply_cgnn}. Differently from the implementation described in Section~\ref{implementation}, we do not pass from the scaling coefficients in $V_j$ to the ones in $V_{j+M}$ and vice-versa at every layer. Here, both the nonlinearity and the convolution are applied to the scaling coefficients in $V_j$. Therefore, our VAE takes as input the scaling coefficients of our signals obtained by doing $6$ downscalings, which already constitute  a significant dimensionality reduction with respect to the original finely discretized signals (the scaling coefficients are approximately $\frac{4096}{2^6} = 64$). We numerically verified that the simplified CGNN architecture provides very similar results to the original architecture. In addition, we also verified that it is less computationally expensive, since the whole network acts only on the scaling coefficients. For our experiments, we choose the stride $s = \frac{1}{2}$, the latent space dimension $\Esse = 15$ and the leaky-ReLU as nonlinear activation function.   

The training is done with the Adam optimizer using a learning rate of $0.01$ and the loss function, commonly used for VAEs, given by the weighted sum of two terms: the Mean Square Error (MSE) between the original and the generated signals and the Kullback-Leibler Divergence (KLD) between the standard Gaussian distribution and the one generated by the encoder in the latent space\footnote{All computations were implemented with Python3, running on a workstation with 256GB of RAM and 2.2 GHz AMD EPYC 7301 CPU and Quadro RTX 6000 GPU with 22GB of memory. All the codes are available at \href{https://github.com/ContGenMod/Continuous-Generative-Neural-Network}
{\texttt{https://github.com/ContGenMod/Continuous-Generative-Neural-Network}}}. 

The injectivity of the decoder, i.e.\ of the generator, is guaranteed if Hypotheses~\ref{hyp_scaling_spaces}, \ref{hyp_filtri} and \ref{hyp_f_c} are satisfied and if $\sigma$ is injective (see Remark~\ref{simply_cgnn}). We test Hypothesis~\ref{hyp_filtri} a posteriori, i.e.\ after the training, and, in our cases, it is always satisfied. The other assumptions are verified using the leaky-ReLU as nonlinearity and the Daubechies scaling functions for the spaces $V_j$. We observed that the leaky-ReLU  empirically outperformed  other
activation functions such as ReLU and ELU.

\subsection{Deblurring with generative models}
\label{sec:deconv}

\subsubsection{The deblurring problem}
We consider the following deblurring problem 
\begin{equation*}
    y = f * x + e,
\end{equation*}
where
\begin{itemize}
    \item $x$ is a smooth signal with support $[0,1]$ discretized with $4096$ equidistant points. More precisely, it is obtained by discretizing a truncated Fourier series:
    \begin{equation*}
        x(t) = \frac{a_0}{2} + \sum_{n=1}^2 a_n \cos{(2 \pi nt)} + b_n \sin({2\pi nt)},
    \end{equation*}
    where $a_n,b_n \sim N\left(0,\frac{1}{(n+1)^6}\right)$. Note that the variance of the Gaussian distribution decreases when the frequency $n$ increases. This ensures smoother signal behavior.
    \item $f$ is the Gaussian blurring filter obtained by evaluating a $N(0,4)$ in $1501$ equidistant point in $[-4,4]$.
    \item $e$ is a weighted $4096$ random Gaussian noise, i.e.\ $e=\tau \epsilon$, where $\tau \geq 0 $ is a weight and $\epsilon \sim \mathcal N(0,I_{4096})$. We consider two different levels of noise corresponding to $\tau = 0$ (no noise, only blurring filter) and $\tau = 0.3$.
    \item The symbol $*$ represents the discrete convolution (see equation \eqref{ast} in Section~\ref{proof_adjoint}).
\end{itemize}

\subsubsection{Reconstruction algorithms}
In order to solve this deblurring problem, we compare the two following approaches.
\paragraph*{Deep Landweber for discrete generators}
We consider the Landweber algorithm applied to $y = f \ast G(z) + e$, where $G$ is a discrete injective generative NN (as described in Section~\ref{sub:1Ddiscrete}) giving as output discretized signals in $\R^{4096}$. In this case, the forward operator is $z\mapsto f \ast G(z)$ and the space in which the iterations are performed is the low-dimensional latent space $\R^{15}$, to which $z$ belongs. In order to choose an appropriate initial value, one option is to generate vectors $z_i \in \R^{15}$ with random Gaussian entries and compute the MSE between the data $y$ and $f \ast G(z_i)$ for every $i$. Then, we  choose the $z_i$ that minimizes the MSE. Otherwise, another option is to choose the initial guess as $E(y)$, where $E$ is the encoder of the VAE. However, we notice that in practice, for our deblurring problem, the algorithm converges to a good solution with almost every initial guess $z \in \R^{15}$ with random Gaussian entries. Therefore, our results are shown in this setting. The iterative step \eqref{deep_landweber} becomes
\begin{equation*}
 z_k = z_{k-1} -h (G'(z_{k-1}))^t f \ast (f \ast G(z_{k-1})-y),\qquad k\ge 1, \qquad h = 0.0005 \text{ fixed}.  
\end{equation*}
We run the algorithm until $\frac{\| z_k-z_{k-1} \|}{\|z_k\|} < \delta$, with $\delta = 10^{-12}$.

\paragraph*{Deep Landweber for continuous generators}
Next, we consider the Landweber algorithm applied to $y = f \ast G(z) + e$, where $G$ is a simplified injective CGNN, as described in Remark~\ref{simply_cgnn}. We consider Daubechies scaling spaces with vanishing moments $N=1,2,6,10$. The generator may be decomposed as $G=\mathcal W\circ \tilde G$, where $\tilde G$ gives as outputs the scaling coefficients of the signals in $V_{j_L}$ and $\mathcal{W}$ is an operator that synthesizes the scaling coefficients at level $j_L$, mapping them to a function in $L^2(\R)$. As for the discrete case above, in practice the algorithm converges to the solution with almost every initial guess $z \in \R^{15}$ with random Gaussian entries and we show the results in this case. The iterative steps are
\begin{equation*}
      z_k = z_{k-1} -h (\tilde G'(z_{k-1}))^t \mathcal{W}^t \circ f \ast (f \ast G(z_{k-1})-y),\qquad k\ge 1, \qquad h = 0.0005 \text{ fixed.}  
\end{equation*}
The stopping criterion is the same as the one of the previous paragraph.

Theoretically, $\mathcal{W}$ maps a sequence of scaling coefficients $(c_n)_{n \in \Z} \in \R^{\Z}$ at level $j=j_L$ into $\sum_{n \in \Z} c_n \phi_{j,n} \in V_j \subset L^2(\R)$ and the corresponding $\mathcal{W}^t$ maps a function $f \in L^2(\R)$ into the sequence of scaling coefficients $(\langle f, \phi_{j,n} \rangle_{L^2(\R)})_{n \in \Z}$.
In practice, as explained in Section \ref{implementation}, we use the scaling coefficients in $V_{j+M}$ for large $M$ to approximate the function $\sum_{n \in \Z} c_n \phi_{j,n}$. Therefore, $\mathcal{W}$ consists of an upsampling transformation repeated $M$ times, assuming all the detail coefficients are equal to zero, while $\mathcal{W}^t$ is an $M$-times downsampling. For our examples, we consider $M=6$.

\subsubsection{Results}

\begin{figure}
\centering
\includegraphics[scale=1]{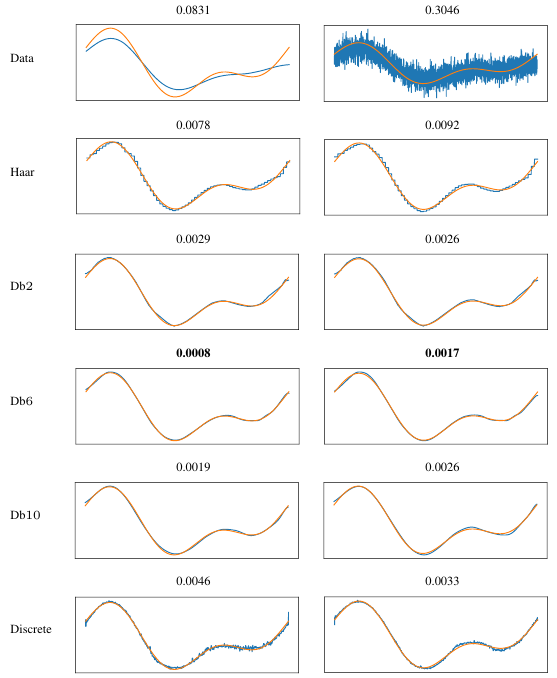}
\caption{Image deblurring. First column: noiseless case. Second column: noisy case. In the first row: original signal $x$ (orange) and blurry one $y$ (blue). In the other rows: original signal $x$ (orange) and reconstructed one (blue) using algorithm \eqref{deep_landweber} where the generators are either CGNNs with different scaling functions or the discrete injective generative NN (last row). At the top of each picture: relative MSE between the two signals.}
\label{deblurring_noise_and_noiseless}
\end{figure}

In Figure~\ref{deblurring_noise_and_noiseless} we provide the results. We compare the original signal (taken from the test set), the corrupted one and the reconstruction, by measuring the relative MSE.

We notice that the discrete VAE yields more irregular reconstructions. This may be due to the significantly higher number of parameters in this network, which is therefore more prone to overfitting. On the contrary, the networks acting on the scaling coefficients at low scales, thanks to the smoothness of Db$6$ and Db$10$, show smoother reconstructions,  coherently with the signals' class. As one would expect, the shape of the wavelet affects the final reconstructions. The choice of the number of vanishing moments $N$ is a tradeoff between the support of the mother wavelet, allowing for good localization properties, and the decay of the wavelet coefficients in the smooth regions. We have not addressed the issue of the optimal choice of $N$ in this work.

In Figure~\ref{deblurring_along_iteration} we show two examples of reconstruction obtained with the iterative algorithm \eqref{deep_landweber} starting from different initial guesses, with the Db$6$ scaling function and blurry data $y$ with no noise. We notice that the algorithm converges to the solution starting from an arbitrary initial guess.
\begin{figure}
\centering
\includegraphics[scale=1]{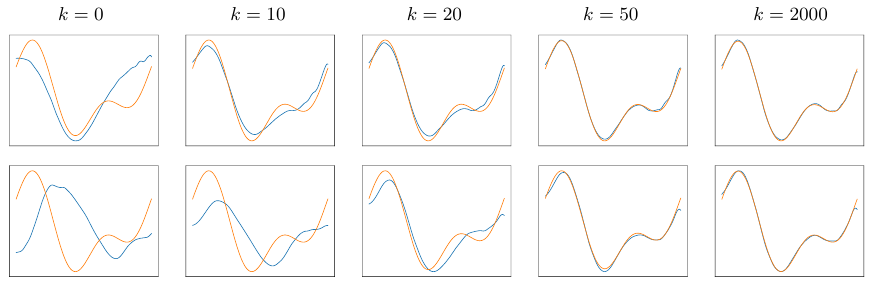}
\caption{In each row: original signal $x$ (orange) and reconstruction (blue) after $k$ iterations of algorithm \eqref{deep_landweber} where the generators are CGNNs with the Db$6$ scaling function. In this case the data is not corrupted by noise; the results with Gaussian noise are almost identical and are omitted.}  
\label{deblurring_along_iteration}
\end{figure}

\subsection{Generation and reconstruction power of CGNNs} 
\label{appendix:generation}

To assess the quality of the generation, in Figure~\ref{generation_comparison} we show random samples of signals from trained injective CGNNs with different Daubechies scaling spaces and from the discrete injective generative NN.

In order to evaluate the reconstruction power of our trained VAEs (Figure~\ref{db1_db2_db6_db10_discrete}), we compute the mean and the variance, over the $200$ signals of the test set, of the MSE between the true signal and the reconstructed one, obtained by applying the full VAE to the true signal. We observe that although the discrete VAE has $50$ times the parameters of the VAEs acting on the scale coefficients, the values of the MSE are comparable (especially in the case of smooth Daubechies wavelets such as db$6$ and db$10$). We also show qualitatively some examples of reconstructed signals in Figure~\ref{ricon_VAE}. Even in these cases, the comparison takes into account the VAEs with different Daubechies scaling spaces and the discrete VAE described in the previous section.

\begin{figure}
\centering
\includegraphics[width=\textwidth]{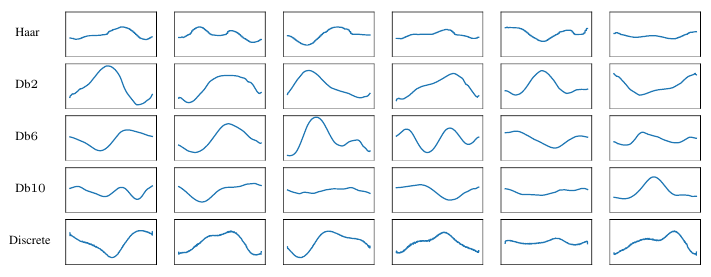}
\caption{Random samples from trained injective CGNNs with different scaling functions and discrete injective generative NN.}
\label{generation_comparison}
\end{figure}

\begin{figure}
\centering  \includegraphics[width=0.5\textwidth]{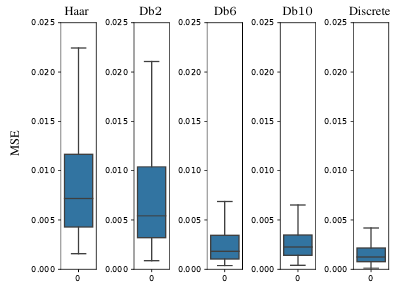}
\caption{Boxplots of the MSE between the original image and the reconstructed one over $200$ images of the test set using VAEs with different scaling functions and the discrete VAE.}
\label{db1_db2_db6_db10_discrete}
\end{figure}

\begin{figure}
\centering    \includegraphics[width=0.9\textwidth]{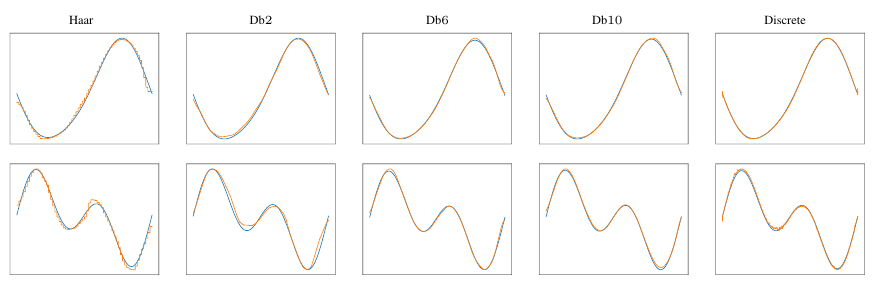}
\caption{Original signal (blue) and reconstructed one (orange) using VAEs with different scaling functions and the discrete VAE.}
\label{ricon_VAE}
\end{figure}

\section{Proofs of the main results}

\subsection{Proof of Theorem~\ref{main_thm}}

\label{injectivity_proof_supp_mat}
We begin with some preliminary technical lemmas.

\begin{lem}
\label{reformulation_eta}
Let Hypothesis~\ref{hyp_scaling_spaces} holds.
Then
\begin{equation*}
\hat{\eta}_1(\xi) \neq 0 \text{ for a.e. } \xi \in [0,1],    
\end{equation*}
where $\hat{\eta}_1(\xi)$ is the Fourier series of $(\eta_1(r))_{r \in \mathbb{Z}}$,  defined as 
\begin{equation}\label{eq:def_fourier}
    \hat{\eta}_1(\xi) := \displaystyle \sum_{r \in \mathbb{Z}} \eta_1(r) e^{-2 \pi i \xi r}, \qquad \xi \in [0,1]. 
\end{equation}
\end{lem}
\begin{proof}
By Hypothesis~\ref{hyp_scaling_spaces}, $\phi$ is compactly supported, thus the series $(\eta_1(r))_{r \in \mathbb{Z}}$ has a finite number of non-zero entries. Then, $\hat{\eta}_1$ is an analytic function, with $\hat{\eta}_1 \not \equiv 0$, since by Hypothesis~\ref{hyp_scaling_spaces}, there exists $r \in \Z$ such that $\eta_1(r) \neq 0$. Therefore $\hat{\eta}_1(\xi) \neq 0 \text{ for a.e. } \xi \in [0,1]$.
\end{proof}

\begin{lem}
\label{shared_compact_support_each_layer}
If Hypotheses~\ref{hyp_scaling_spaces}, \ref{hyp_filtri} and \ref{hyp_f_c} are satisfied, then the image of each layer of the generator $G$ defined in \eqref{continuous_generator} contains only compactly supported functions. More precisely: 
\begin{equation*}
    \left( \comp_{\tilde{l}=l}^2 \tilde{\sigma}_{\tilde{l}} \circ \tilde\Psi_{\tilde{l}}\right)
    \circ \left(\tilde{\sigma}_1 \circ \Psi_1\right)(\R^{\Esse}) \subset W_{l}, \quad l=2,...,L, \qquad \left(\tilde{\sigma}_1 \circ \Psi_1\right)(\R^{\Esse}) \subset W_{1},
\end{equation*}  
\begin{equation*}
    \tilde\Psi_{l+1} \circ \left( \comp_{\tilde{l}=l}^2 \tilde{\sigma}_{\tilde{l}} \circ \tilde\Psi_{\tilde{l}}\right)
    \circ \left(\tilde{\sigma}_1 \circ \Psi_1\right)(\R^{\Esse}) \subset {W}_{l+1}, \quad l=2,...,L-1, \qquad \tilde\Psi_{2} \circ \left(\tilde{\sigma}_1 \circ \Psi_1\right)(\R^{\Esse}) \subset {W}_{2},
\end{equation*}
where 
\[
W_{l} := (\spann \{ \phi_{j_{l},n} \}_{n = -N^{l}}^{N^{l}})^{c_{l}},\qquad l=1,\dots,L,
\]
for some $N^{l} \in \N$.
\end{lem}
\begin{proof}
By Hypothesis~\ref{hyp_filtri}, the filters of each layer are compactly supported  and, by Hypothesis~\ref{hyp_f_c}, the same happens to the functions in the image of the first (fully connected) layer, $F(\mathbb{R}^{\Esse}) + b$. Moreover, the nonlinearities do not change the support of the functions since they act pointwisely, and each $V_j$ is spanned by the translates of a fixed compactly supported function. Therefore, the image of each layer contains only compactly supported functions.
\end{proof}

In the rest of this section, with an abuse of notation, we indicate with $\sigma$ both the scalar function $\sigma \colon \R \to \R$ and the operator $\sigma \colon V_j^c \to (L^2(\R))^c$. 

\begin{lem}
\label{derivata_di_sigma}
Let $\sigma \colon \R \to \R$ satisfy  Hypothesis~\ref{hyp_sigma}. Let $W$ be a finite-dimensional subspace of $(L^2(\R))^c$ of the form $W_1 \times ... \times W_c$, where $W_i$ are finite-dimensional subspaces of $L^2(\R) \cap L^{\infty}(\R)$ for every $i=1,...,c$.
Then $\sigma \colon W \subset L^2(\R)^c\to L^2(\R)^c$ is Fr\'echet differentiable and its  Fr\'echet derivative  in $g \in W$ is
\begin{equation*}
\begin{aligned}
\sigma'(g) \colon W &\to (L^2(\R))^c\\
f &\mapsto \sigma'(g)[f]
\end{aligned}
\end{equation*}
where
\begin{equation*}
\begin{aligned}
(\sigma'(g)[f])_i \colon \R &\to \mathbb{R}\\
x &\mapsto \sigma ' (g_i(x)) f_i(x),
\end{aligned}
\end{equation*}
for every $i=1,...,c$. Moreover, $\sigma \in C^1(W)$.
\end{lem}
\begin{proof}
Without loss of generality, set $c=1$. Let $g \in W$.
First, we observe that $\sigma'(g)$, as defined above, is linear. It is also bounded i.e.\ there exists $C \in \mathbb{R}^+$ such that $\| \sigma'(g)[f]\|^2_2 \leq C \| f\|^2_2$. Indeed, by  Hypothesis~\ref{hyp_sigma} we have
\begin{equation*}
\| \sigma'(g)[f]\|^2_2  := \int_{\R} ({\sigma}'(g(x))f(x))^2 dx \\  \leq M_2^2 \int_{\R} f(x)^2 dx =M_2^2\| f\|^2_2. 
\end{equation*}

Next, we show that $\sigma'(g)$ is the Fr\'echet derivative of $\sigma$, namely
\begin{equation*}
\lim_{t \to 0} \sup_{\|f\|_2 = 1, f \in W} \frac{\|\sigma(g+tf) - \sigma(g) - t \hspace{0.1cm} \sigma'(g)[f] \|_{2}}{t} = 0,
\end{equation*}
which is equivalent to
\begin{equation}
\label{f_d_quadro}
	    \lim_{t \to 0} \sup_{\|f\|_2 = 1, f \in W} \frac{\|\sigma(g+tf) - \sigma(g) - t \hspace{0.1cm} \sigma'(g)[f] \|^2_{2}}{t^2} = 0.
\end{equation}
Indeed, fixing $x \in \R$, since $\sigma \in C^1(\mathbb{R})$, the mean value theorem yields
\begin{equation*}
\sigma(g(x)+t f(x)) - \sigma(g(x)) = t f(x) \sigma'(g(x) + \tau f(x)),
\end{equation*}
where $0 \leq \tau \leq t \leq 1$.
Then,  \eqref{f_d_quadro} becomes
\begin{equation*}
	\lim_{t \to 0} \sup_{\|f\|_2 = 1, f \in W} \int_{\R}   f(x)^2 (\sigma'(g(x)) - \sigma'(g(x) + \tau f(x)))^2 dx. 
\end{equation*}
We observe that the space where $\sigma'$ is evaluated is
\begin{equation*}
    \{ g(x)+\tau f(x) : x \in \R, 0 \leq \tau\leq 1, f \in W \text{ s.t. } \| f \|_2 = 1 \},
\end{equation*} 
which is contained in $K=[ -(\| g \|_{\infty} +  \tilde{C}), \| g \|_{\infty} +  \tilde{C}]$ where $\| g \|_{\infty}$ is finite since $g \in W\subset L^\infty(\R)$, and $\tilde{C}=\sup_{f \in W, \|f\|_2 =1} \|f\|_\infty$, which is finite since $ W$ is a  finite-dimensional subspace of $L^2(\R) \cap L^{\infty}(\R)$, therefore $\| \cdot \|_2$ is equivalent to $\| \cdot\|_{\infty}$ in $W$.

Moreover, $\sigma'$ is uniformly continuous in $K$, because $\sigma \in C^1(\R)$ by Hypothesis~\ref{hyp_sigma}. Therefore, there exists a modulus of continuity $\omega \colon \R^+ \to \R$ such that 
\begin{equation*}
    | \sigma'(x) - \sigma'(y) | \leq \omega(| x-y|), \qquad x,y \in K,
\end{equation*}
and
\begin{equation}\label{eq:omega0}
    \lim_{t \to 0^+} \omega(t) = \omega(0) = 0.    
\end{equation}
We notice also that we can choose $\omega$ as an increasing function.
Therefore, we obtain 
\begin{equation*}
\begin{split}
    & \lim_{t \to 0} \sup_{\|f\|_2 = 1, f \in W} \int_{\R}   f(x)^2 (\sigma'(g(x)) - \sigma'(g(x) + \tau f(x)))^2 dx \\ 
    & \leq \lim_{t \to 0} \sup_{\|f\|_2 = 1, f \in W}  \int_{\R} f(x)^2 \omega(|\tau f(x)|)^2 dx \\ 
    & \leq \lim_{t \to 0} \sup_{\|f\|_2 = 1, f \in W}  \int_{\R} f(x)^2\omega( t \|  f \|_{\infty})^2 dx\\
    & \leq \lim_{t \to 0}   \omega( t \tilde C)^2 \\
    & = 0.
\end{split}
\end{equation*}

Finally, we show that $\sigma \in C^1(W)$. The continuity of $\sigma'$ in $g$ is verified if
    \begin{equation*}
    \lim_{\| f \|_2 \to 0, f \in W} \| \sigma'(g+f) - \sigma'(g) \|^2_{\mathcal{L}(W,L^2(\R))} = 0.    
    \end{equation*}
    We have
    \begin{equation*}
    \begin{split}
       &  \| \sigma'(g+f) - \sigma'(g) \|^2_{\mathcal{L}(W,L^2(\R))} \\ &  := \sup_{\| h \|_2 = 1, h \in W} \| \sigma'(g+f)[h] - \sigma'(g)[h] \|^2_2 \\ & =
        \sup_{\| h \|_2 = 1, h \in W}  \int_{\R} h(x)^2 (\sigma'(g(x)+f(x))- \sigma'(f(x)))^2 dx.
    \end{split}
    \end{equation*}
    We consider $\| f \|_2 \leq 1$, since $\| f \|_2 \to 0$. Using an argument similar to the one above, we observe that the space where $\sigma'$ is evaluated is compact in $\R$. Thus, there exists an increasing modulus of continuity, as explained before.  
    Therefore, we obtain
    \begin{equation*}
    \begin{split}    
        & \sup_{\| h \|_2 = 1, h \in W}  \int_{\R} h(x)^2 (\sigma'(g(x)+f(x))- \sigma'(f(x)))^2 dx \\ & 
        \leq \sup_{\| h \|_2 = 1, h \in W}  \int_{\R} h(x)^2 \omega(|f(x)|)^2 dx \\ &
         \leq  \omega(\|f\|_{\infty})^2,
    \end{split}
    \end{equation*}
    and $  \omega(\|f\|_{\infty})^2 \to 0$ as $\| f \|_2 \to 0$ by \eqref{eq:omega0}, since $\| \cdot \|_2$ and $\| \cdot \|_{\infty}$ are equivalent norms in $W$.
\end{proof}

We recall a classical result that will be used in the proof: Hadamard's global inverse function theorem.
\begin{thm}[\cite{Gordon_diffeomorfismo}]
\label{Hadamard_thm}
A $C^1$ map $f\colon \R^N\to\R^N$ is a diffeomorphism if and only if the Jacobian never vanishes, i.e.\ $\det(J_f(x)) \neq 0$ for every $x \in \R^N$, and $| f(x) | \to +\infty$ as $| x | \to +\infty$.
\end{thm}

We are now able to prove the main theorem of the paper.

\begin{proof}[Proof of Theorem~\ref{main_thm}]
Let us recall the network $G$ in \eqref{continuous_generator}:
\begin{equation*}
    G= \left( \comp_{l=L}^2 \tilde{\sigma}_{l} \circ \tilde\Psi_l\right)
    \circ \left(\tilde{\sigma}_1 \circ \Psi_1 \right),
\end{equation*}
where
\begin{align*}
 &\tilde{\Psi}_{l} := P_{(V_{j_{l}})^{c_{l}}} \circ \Psi_{l}\colon (V_{j_{l-1}})^{c_{l-1}}\to (V_{j_{l}})^{c_{l}}, \qquad l=2,...,L,\\
&\tilde{\sigma}_{l} := P_{(V_{j_{l}})^{c_{l}}} \circ \sigma_l \colon (V_{j_{l}})^{c_{l}} \to (V_{j_{l}})^{c_{l}}, \qquad l=1,...,L.  
\end{align*}

Note that $\Psi_1=F\cdot +b$ is injective by Hypothesis~\ref{hyp_f_c}. We will show that the restriction of $\tilde\Psi_l$ to the image of the previous layer is injective for every $l=2,\dots,L$ and that the restriction of $\tilde{\sigma}_l$ to the image of the previous layer is injective for every $l=1,\dots,L$. The injectivity of $G$ will immediately follow.

Let us fix a layer $l$. The proof holds for every $l=2,...,L$ for $\tilde\Psi_l$ and for every $l=1,...,L$ for $\tilde{\sigma}_l$.
To simplify the notation, we omit the dependence of the quantities on $l$ and we denote  the number of input channels by  $c$ and  the input scale by $j$. 

\paragraph*{Step 1: \textit{Injectivity of $\tilde\Psi \big|_{W}$}} Let $W = W_{l-1}$, where $W_{l-1}$ is defined in Lemma~\ref{shared_compact_support_each_layer}.
Therefore, there exists $N \in \N$ such that $W = (\spann \{ \phi_{j,n} \}_{n=-N}^{N})^c$. There are $\frac{c}{2}$ output channels and the output scale is $j+1$.
Take $h, g \in W$ such that $\tilde{\Psi}(h) = \tilde{\Psi}(g)$, i.e.\
\begin{equation*}
P_{V_{j+1}} \Big( \sum_{i=1}^{c} t_{i,k} \ast h_i + b_k \Big) = P_{V_{j+1}} \Big( \sum_{i=1}^{c} t_{i,k} \ast g_i + b_k \Big),\quad k=1,...,\frac{c}{2}.
\end{equation*}
We need to show that $h=g$. By the linearity of the projections and of the convolutions, we obtain
\begin{equation}
\label{kernel_convoluzione_nullo}
P_{V_{j+1}} \Big( \sum_{i=1}^{c} t_{i,k} \ast f_i \Big) = 0,\quad k=1,...,\frac{c}{2},
\end{equation}
where $f=h-g$. We need to show that $f_i=0$ for every $i = 1,...,c$. 

Recall that \{$t_{i,k}\}_{i,k}$ satisfy Hypothesis~\ref{hyp_filtri}, and $f \in W$, i.e.\
\begin{equation*}
t_{i,k} = \sum_{p= 0}^{\bar{p}} d_{p,i,k} \phi_{j+1,p}, \qquad d_{p,i,k} = \langle t_{i,k}, \phi_{j+1,p} \rangle_2,
\end{equation*}
and
\begin{equation*}
f_i = \sum_{n=-N}^{N} v_{n,i} \phi_{j,n}, \qquad v_{n,i} = \langle f_i, \phi_{j,n} \rangle_2.
\end{equation*}
Since $\{ \phi_{j+1,m} \}_{m \in \mathbb{Z}}$ is an orthonormal basis of $V_{j+1}$, we reformulate \eqref{kernel_convoluzione_nullo} in the following way 
\begin{equation}
\label{kernel_convoluzione_nullo_2}
    \sum_{i=1}^{c} \sum_{n = -N}^{N} \Big( \sum_{p = 0}^{\bar{p}} {d}_{p,i,k} \langle {F}_{n,p} , \phi_{j+1,m} \rangle_2 \Big) v_{n,i} = 0, \quad k=1,...,\frac{c}{2}, \quad m \in \mathbb{Z},
\end{equation}
where ${F}_{n,p} := \phi_{j,n} \ast \phi_{j+1,p}$. We need to show that $v_{n,i} = 0$ for every $i = 1,...,c$ and every $n = -N,...,N$. Setting
\begin{equation}
\label{A}
A_{(m,k),(n,i)} := \sum_{p =0}^{\bar{p}} {d}_{p,i,k} \langle {F}_{n,p} , \phi_{j+1,m} \rangle_2,     
\end{equation}
equation~\eqref{kernel_convoluzione_nullo_2} becomes
\begin{equation}
\label{kernel_convoluzione_nullo_3}
    \sum_{i=1}^{c} \sum_{n =-N}^N A_{(m,k),(n,i)} v_{n,i} = 0, \quad k=1,...,\frac{c}{2}, \quad m \in \mathbb{Z}.
\end{equation}
We observe that $\langle {F}_{n,p} , \phi_{j+1,m} \rangle_2$ depends only on the scaling function $\phi$, the scale $j$ and the coefficient $m-p-2n$ as show in \eqref{eta_compatto}, i.e.\ $\langle {F}_{n,p} , \phi_{j+1,m} \rangle_2 = 2^{-\frac{j}{2}} \eta_1(m-p-2n)$ with $\eta_1$ defined in \eqref{eta}.
Since $\phi$ is compactly supported, $\eta_1$ has a finite number of non-zero entries, namely $\eta_1 \in c_{00}(\Z)$, and, by Hypothesis~\ref{hyp_filtri}, the same holds for $d_{\cdot,i,k}$ (suitably extended by $0$ outside its support).
This allows us to rewrite \eqref{A} as 
\begin{equation*}
A_{(m,k),(n,i)} = 2^{-\frac{j}{2}} \sum_{p \in \Z} {d}_{p,i,k} \eta_1(m-p-2n) = 2^{-\frac{j}{2}} ({d}_{\cdot,i,k} \ast \eta) (m-2n) = 2^{-\frac{j}{2}} \hspace{0.05cm} d^\eta_{i,k}(m-2n),   
\end{equation*}
where $d^\eta_{i,k} := ({d}_{\cdot,i,k} \ast \eta_1) \in c_{00}(\Z)$ and $\ast$ represents the discrete convolution defined in \eqref{ast}. 

Then \eqref{kernel_convoluzione_nullo_3} is equivalent to
\begin{equation*}
\sum_{i=1}^{c} \sum_{n \in \mathbb{Z}} d^\eta_{i,k}(m-2n) v_{n,i} = 0, \qquad m \in \mathbb{Z}, \quad k=1,...,\frac{c}{2},
\end{equation*}
where $v_{\cdot,i} \in c_{00}(\Z)$.
Therefore
\begin{equation}
\label{kernel_fourier}
\sum_{i=1}^{c} \sum_{n,m \in \mathbb{Z}} d^\eta_{i,k}(m-2n) v_{n,i} e^{-2\pi i \xi m}= 0, \qquad \text{ a.e.\ } \xi \in [0,1].
\end{equation}
Rewriting \eqref{kernel_fourier} as 
\begin{equation*}
    \sum_{i=1}^{c} \sum_{n \in \mathbb{Z}} \bigg( \sum_{m \in \mathbb{Z}} d^\eta_{i,k}(m-2n) e^{-2\pi i \xi (m-2n)} \bigg) v_{n,i} e^{-2\pi i \xi 2n} = 0
\end{equation*}
and defining the Fourier series of $(g(r))_{r \in \mathbb{Z}} \in \ell^2(\Z)$ as in \eqref{eq:def_fourier}, we obtain
\begin{equation*}
    \sum_{i=1}^{c} \widehat{d^\eta_{i,k}}(\xi) \widehat{v_i}(2\xi)= 0, \qquad \text{ a.e.\ } \xi \in [0,1],
\end{equation*}
where $v_i(n) := v_{n,i}$. Applying the Convolution Theorem, we obtain
\begin{equation}
\label{kernel_fourier_2}
\sum_{i=1}^{c} {\hat{\eta}_1}(\xi) \widehat{d_{i,k}}(\xi) \widehat{v_i}(2\xi)= 0, \qquad \text{ a.e.\ } \xi \in [0,1],   
\end{equation}
where $d_{i,k}(p) := d_{p,i,k}$.

Thanks to Lemma~\ref{reformulation_eta}, equation~\eqref{kernel_fourier_2} is equivalent to 
\begin{equation*}
\sum_{i=1}^{c} \widehat{d_{i,k}}(\xi) \widehat{v_i}(2\xi)= 0, \qquad \text{ a.e.\ } \xi \in [0,1], \qquad k = 1,...,\frac{c}{2}.
\end{equation*}  
We can rewrite this condition on the coefficients, by writing the definition of the Fourier series and using the orthonormality of $\{ e^{2\pi i \xi r} \}_{r \in \mathbb{Z}}$, which gives 
\begin{equation}
\label{non_splittata}
\sum_{i=1}^{c} \sum_{n \in \mathbb{Z}} d_{m-2n,i,k} v_{n,i} = 0, \qquad m \in \mathbb{Z}, \qquad k = 1,...,\frac{c}{2}.
\end{equation} 
Considering  the odd and the even entries of ${d}_{\cdot,i,k}$ separately, we can split \eqref{non_splittata} in this way 
\begin{equation}
\label{kernel_splittato}
\sum_{i=1}^{c} \sum_{n \in \mathbb{Z}} {\tilde{d}}_{m-n,i,k} v_{n,i} = 0, \qquad m \in \mathbb{Z}, \qquad k = 1,...,c,
\end{equation}
where 
\begin{equation*}
{\tilde{d}}_{p,i,k} :=
    \begin{cases}
      d_{2p,i,k} & k = 1,...,\frac{c}{2},\\
      d_{2p+1,i,k-\frac{c}{2}} & k = \frac{c}{2}+1,...,c.
    \end{cases}
\end{equation*}
We notice that in \eqref{kernel_splittato} we have $k=1,...,c$, while previously $k=1,...,\frac{c}{2}$. This is due to the splitting operation that \textit{doubles} the equations \eqref{non_splittata}. Now the indices $i$ and $k$ take values in the same range.

Without loss of generality, let us assume $\bar{p}$ odd. Using again Hypothesis~\ref{hyp_filtri}, we have that ${\tilde{d}}_{p,i,k} = 0$ if $p \notin \{ 0,...,\frac{\bar{p} - 1}{2} \}$ and $v_{n,i} = 0$ if $n \notin \{ -N,...,N \}$. So, we can rewrite \eqref{kernel_splittato} in a compact form as ${\tilde{D}} v = 0$, where ${\tilde{D}}$ is a matrix of size $c(2 N+\frac{\bar{p}+1}{2}) \times c(2 N+1)$ of the form
\begin{equation}
\label{D_tilde}
{\tilde{D}} = 
\begin{pmatrix}
{\tilde{D}}_{0} & \cdots & 0 & \cdots & 0 & \cdots & 0 \\
\vdots  & \ddots & \vdots &  & \vdots &  & \vdots  \\
{\tilde{D}}_{\frac{\bar{p}-1}{2}} & \cdots & {\tilde{D}}_{0} & \cdots & 0 & \cdots & 0\\
\vdots  & \ddots & \vdots & \ddots & \vdots &  & \vdots  \\
0 & \cdots & {\tilde{D}}_{\frac{\bar{p}-1}{2}} & \cdots & {\tilde{D}}_{0} & \cdots & 0 \\
\vdots &  & \vdots & \ddots & \vdots& \ddots & \vdots \\
0 & \cdots & 0 & \cdots & {\tilde{D}}_{\frac{\bar{p}-1}{2}} & \cdots & {\tilde{D}}_{0} \\
\vdots &  & \vdots & & \vdots& \ddots & \vdots \\
0 & \cdots & 0 & \cdots & 0 & \cdots & {\tilde{D}}_{\frac{\bar{p}-1}{2}} \\
\end{pmatrix},
\end{equation}
where ${\tilde{D}}_{p}$ is the $c \times c$ matrix defined by $({\tilde{D}}_{p})_{i,k} := {\tilde{d}}_{p,i,k}$, and $v$ is a vector of size $c(2N+1)$ such that
\begin{equation*}
    v = \begin{pmatrix}
v_{-N} \\
\vdots \\
v_N \\
\end{pmatrix},
\end{equation*}
where $v_n$ is a vector of size $c$ such that $(v_n)_{i} := v_{n,i}$. To clarify the block structure of ${\tilde{D}}$, the $(m,n)$ block of size $c \times c$ is ${\tilde{D}}_{m-n}$, defined as $({\tilde{D}}_{m-n})_{i,k} := {\tilde{d}}_{m-n,i,k}$, if $0 \leq m-n \leq \frac{\bar{p}-1}{2}$ and the zero matrix otherwise, where $n=-N,...,N$ and $m=-N,...,N+\frac{\bar{p}-1}{2}$. 
We observe that the matrix $D^l$ defined in \eqref{D_0} corresponds to $\tilde{D}_{0}$ defined above and by Hypothesis~\ref{hyp_filtri} $\det(D^l) = \det(\tilde{D}_{0}) \neq 0$. Therefore, the rank of ${\tilde{D}}$ is maximum, since the determinant of the $c(2N+1) \times c(2N+1)$ block-triangular matrix 
\begin{equation*}
\begin{pmatrix}
{\tilde{D}}_{0} & \cdots & 0 & \cdots & 0 & \cdots & 0 \\
\vdots  & \ddots & \vdots &  & \vdots &  & \vdots  \\
{\tilde{D}}_{\frac{\bar{p}-1}{2}} & \cdots & {\tilde{D}}_{0} & \cdots & 0 & \cdots & 0\\
\vdots  & \ddots & \vdots & \ddots & \vdots &  & \vdots  \\
0 & \cdots & {\tilde{D}}_{\frac{\bar{p}-1}{2}} & \cdots & {\tilde{D}}_{0} & \cdots & 0 \\
\vdots &  & \vdots & \ddots & \vdots& \ddots & \vdots \\
0 & \cdots & 0 & \cdots & {\tilde{D}}_{\frac{\bar{p}-1}{2}} & \cdots & {\tilde{D}}_{0} \\
\end{pmatrix}
\end{equation*}
is not zero. Hence, $v_{n,i} = 0$ for every $n \in \Z$ and $i=1,...,c$. \\

\paragraph*{Step 2: \textit{Injectivity of $\tilde{\sigma} \big|_{W}$}} Let $W = W_{l}$, where $W_{l}$ is defined in Lemma~\ref{shared_compact_support_each_layer}.
Therefore, there exists $N \in \N$ such that $W = (\spann \{ \phi_{j,n} \}_{n=-N}^{N})^c$.
We prove that $\sigma_W := P_{W} \circ \sigma \big|_{W} \colon W \to W$ is injective. This implies that $\tilde{\sigma} \big|_{W} = P_{(V_j)^c} \circ \sigma \big|_{W} \colon W \to (V_j)^c$ is injective. Without loss of generality, set $c=1$. 

To do this, we use Theorem~\ref{Hadamard_thm}. By Hypothesis~\ref{hyp_scaling_spaces}, we have $W \subset L^2(\R) \cap L^{\infty}(\R)$. Thanks to Lemma~\ref{shared_compact_support_each_layer}, we can identify any function $f \in W$ with a vector $y \in \mathbb{R}^{2N+1}$ whose entries are $\langle f, \phi_{j,n} \rangle_2$ for $n=-N,...,N$. 
Therefore, $\sigma_W$ can be seen as a map
\begin{equation*}
\sigma_W \colon \mathbb{R}^{2N+1} \to \mathbb{R}^{2N+1}.
\end{equation*}

To prove the injectivity of $\sigma_W$, we verify the conditions of Theorem~\ref{Hadamard_thm}.

\begin{enumerate}
    \item The function $\sigma_W \in C^1(W)$, because it is the composition of a linear function, the projection $P_{W}$, and the nonlinearity, $\sigma$, which is of class $C^1$ thanks to Lemma~\ref{derivata_di_sigma}. 
    \item We now show that $\det(J_{\sigma_W}(y)) \neq 0$. Using $g \in W$ instead of the corresponding vector $y \in \mathbb{R}^{2N+1}$, we observe that
    \begin{equation*}
    J_{\sigma_W}(g) = (P_{W} \circ \sigma') (g),   
    \end{equation*}
    thanks to the linearity of the projection. In order to show that
\begin{equation*}
\det(J_{\sigma_W}(g)) \neq 0  
\end{equation*}
for every $g \in W$, by Lemma~\ref{derivata_di_sigma} we need to prove that $P_{W}(\sigma'(g(\cdot)) f) = 0$ a.e.\ implies $f = 0$ a.e. Thanks to the fact that $f \in W$, we have
\begin{equation*} 
 \langle \sigma'(g(\cdot)) f, f \rangle_2 = 0.
\end{equation*}
But $\sigma'(x) > 0$ for every $x \in \mathbb{R}$, and so the last equation implies $f = 0$ a.e.
    \item We now verify that $|\sigma_W(y)|_{\mathbb{R}^{2N+1}} \to + \infty$ as $|y |_{\mathbb{R}^{2N+1}} \to + \infty$. We observe that $|y |_{\mathbb{R}^{2N+1}} \to + \infty$ implies that $\| f \|_2 \to + \infty$,  where $f \in W$ has been identified with $y$. 
        Moreover, 
\begin{equation*}
    \| \sigma_W(f) \|_2 = \| (P_{W} \circ \sigma)(f) \|_2 \geq \langle \sigma(f), \frac{f}{\| f \|_2} \rangle_2,
\end{equation*}
because $\frac{f}{\| f \|_2} \in W$ with unitary norm. Using the fact that $x \cdot \sigma(x) \geq 0$ for every $x \in \R$ (see Condition~\ref{cond_2} of Hypothesis~\ref{hyp_sigma}), we have
\begin{equation*}
    \langle \sigma(f), \frac{f}{\| f \|_2} \rangle_2 = \frac{1}{\| f \|_2} \int_{\mathbb{R}} \sigma(f(x)) f(x) dx = \frac{1}{\| f \|_2} \int_{\mathbb{R}} |\sigma(f(x))| |f(x)| dx.
\end{equation*}

Consider $x \geq 0$. Applying the mean value theorem to $\sigma$ in the interval $[0, x]$ we have that $\sigma(x) = \sigma'(\xi) x$ with $\xi \in [0,x]$. Thanks to the fact that $\sigma'(x) \geq M_1$ for every $x \in \R$ (see Condition~\ref{cond_1} of Hypothesis~\ref{hyp_sigma}), we obtain that $|\sigma(x)| \geq M_1 |x|$ for every $x \geq 0$. The same holds for $x < 0$. Therefore, 
\begin{equation*}
\frac{1}{\| f \|_2} \int_{\mathbb{R}} |\sigma(f(x))| |f(x)| dx \geq \frac{M_1}{\| f \|_2} \int_{\R} (f(x))^2 dx = M_1 \| f \|_2.
\end{equation*}
Summing up, we have
\begin{equation*}
    \|\sigma_W(f) \|_2 \geq M_1 \| f \|_2 \to + \infty, \qquad \| f \|_2 \to + \infty,
\end{equation*}
and then we obtain the thesis.
\end{enumerate}
\end{proof}

\begin{remark}
    As we can see in the first step of the proof, the second condition in Hypothesis~\ref{hyp_filtri} is sufficient but not necessary. It is indeed enough to ask that the rank of $\tilde{D}$, defined in \eqref{D_tilde}, is maximum.
  
    Moreover, we observe that the maximum rank condition may not hold if we have $\det(\tilde{D}_{0}) = 0$, but $\det(\tilde{D}_{p}) \neq 0$ for some $p \in \{ 1,...,\bar{p}-1 \}$. Indeed, if we consider the matrix
    \begin{equation*}
    \begin{pmatrix}
    A & 0 & 0 \\
    I & A & 0 \\
    A & I & A \\
    0 & A & I \\
    0 & 0 & A \\
    \end{pmatrix}
    \end{equation*}
    where $0 \in \mathcal{M}_{2 \times 2}$ is the zero matrix, $I \in \mathcal{M}_{2 \times 2}$ is the identity matrix and $A \in \mathcal{M}_{2 \times 2}$ is such that
    \begin{equation*}
    \begin{pmatrix}
    1/\sqrt{2} & 0 \\
    0 & 0
    \end{pmatrix},
    \end{equation*}
    it is easy to verify that the rank is not maximum.
\end{remark}

\subsection{Proof of Theorem~\ref{lipschitz_stab}}\label{app:lip}

We first prove that the Fr\'echet derivative of an injective CGNN is injective as well.

\begin{prop}
\label{differentiable_manifold}
Let $X = L^2(\R)$ and $G$ be a CGNN satisfying the hypotheses of Theorem~\ref{main_thm}.
Then the generator $G$ is injective, of class $C^1$ and $G'(z)$ is injective for every $z \in \mathbb{R}^{\Esse}$.
\end{prop}

\begin{proof}
The injectivity of $G$ is proved in Theorem~\ref{main_thm}, and the continuous differentiability of $G$ follows from the fact that it is a composition of continuously differentiable functions. We only need to prove the injectivity of $G'$. Let $W_{l}$ be defined as in Lemma~\ref{shared_compact_support_each_layer}, for $l=1,\dots,L$.
The derivative of $G$ can be written as
	\begin{equation} \label{G'}
	G' = \left(\comp_{l=L}^2 (P_{(V_{j_{l}})^{c_{l}}} \circ \sigma_{l})' (P_{(V_{j_{l}})^{c_{l}}} \circ \Psi_l)' \right) (P_{(V_{j_1})^{c_1}} \circ \sigma_1)' F',
	\end{equation}
	where, for simplicity, we omitted the arguments of each term.
    The injectivity of $G'(z)$ for every $z \in \mathbb{R}^{\Esse}$ follows from the injectivity of each component of \eqref{G'}. Indeed, $F' = F$ is injective for every $z \in \R^{\Esse}$ by Hypothesis~\ref{hyp_f_c}. 
    Moreover, for every $l=2,...,L$, $(P_{(V_{j_{l}})^{c_{l}}} \circ \Psi_l)' \big|_{{W}_{l-1}} = P_{(V_{j_{l}})^{c_{l}}} \circ \Psi_{l} \big|_{{W}_{l-1}}$ is injective, as we prove in Step $1$ of the proof of Theorem~\ref{main_thm}. 
    
 Finally, we prove that $\tilde{\sigma}'_l \big|_{{W}_{l}}(g) = (P_{(V_j)^c} \circ \sigma_l \big|_{{W}_{l}})'(g)$ is injective for every $l=1,...,L$. To do this, we observe that a stronger condition holds i.e.\ $\sigma'_{{W}_{l}}(g) := (P_{{W}_{l}} \circ \sigma_l\big|_{{W}_{l}})'(g)$ is injective for every $l=1,...,L$. Indeed, the deteminant of its Jacobian is not zero for every $g \in {W}_{l}$, as shown in Step $2$ of the proof of Theorem~\ref{main_thm}.
\end{proof}

We are then able to prove a Lipschitz stability estimate for \eqref{IPG}, when $G$ is an injective CGNN.

\begin{proof}[Proof of Theorem~\ref{lipschitz_stab}]
Thanks to Proposition~\ref{differentiable_manifold}, $G$  is injective, of class $C^1$ and $G'(z)$ is injective for every $z \in \mathbb{R}^{\Esse}$. Therefore,  $\mathcal{M} = G(\R^{\Esse})$  is  a $\Esse$-dimensional differentiable manifold embedded in $X = L^2(\R)$, considering as atlas the one formed by only one chart $\{ \mathcal{M}, G^{-1} \}$. Moreover, $\mathcal{M}$ is a Lipschitz manifold, since $G$ is a composition of Lipschitz maps. Indeed, it is composed by affine maps and Lipschitz nonlinearities (see Hypothesis~\ref{hyp_sigma}). Then the result immediately follows from  \cite[Theorem $2.2$]{Alberti_Santacesaria_Arroyo}.
\end{proof}

The same result can be obtained also in the case of non-expansive convolutional layers and arbitrary stride, under the hypotheses of Theorem~\ref{cor_expansive_stride}. Moreover, it can be extended to the $2$D case under the hypotheses of Theorem~\ref{cor_2D}.

\section{Conclusions}

In this work, we have introduced CGNNs, a family of generative models in the continuous, infinite-dimensional, setting, generalizing popular architectures such as DCGANs \cite{DCGAN}. We have shown that, under natural conditions on the weights of the networks and on the nonlinearity, a CGNN is globally injective. This allowed us to obtain a Lipschitz stability result for (possibly nonlinear) ill-posed inverse problems, with unknowns belonging to the manifold generated by a CGNN.

The main mathematical tool used is wavelet analysis and, in particular, a multi-resolution analysis of $L^2(\R)$. While wavelets yield the simplest multi-scale analysis, they are suboptimal when dealing with images. So, it would be interesting to consider CGNNs with other systems, such as curvelets \cite{candes2004new} or shearlets \cite{labate2005sparse}, more suited to higher-dimensional signals, or, more generally, CGNNs made of an injective discrete neural network composed with a synthesis operator not necessarily associated to a multi-scale structure.

For simplicity, we considered only the case of a smooth nonlinearity $\sigma$: we leave the investigation of Lipschitz $\sigma$'s to future work. This would allow for including many commonly used activation functions. 
Some simple illustrative numerical examples are included in this work, which was mainly focused on the theoretical properties of CGNNs. It would be interesting to perform more extensive numerical simulations in order to better evaluate the performance of CGNNs, also  with nonlinear inverse problems, such as electrical impedance tomography.

The stability estimate obtained in Theorem~\ref{lipschitz_stab} is based on a more general result that holds also with manifolds that are not described by one single generator \cite{Alberti_Santacesaria_Arroyo}. A possible direction to extend our work is combining multiple generators to learn several charts of a manifold, as done in \cite{alberti2023manifold} for the finite-dimensional case.

The output of a CGNN is always in a certain space $V_j$ of a MRA and, as such, is as smooth as the scaling function. As a result, this architecture is in general not well-suited for classes of irregular signals. It would be very interesting to design and study architectures  that can generate more irregular, e.g. discontinuous, signals in functions spaces. This would probably require exploiting the full depth of the wavelet MRA, in order to capture arbitrarily fine scales, and  keeping  at the same time a low-dimensional latent space, or going beyond multi-scale representations of signals.

\section*{Acknowledgments} 
This material is based upon work supported by the Air Force Office of Scientific Research under award number FA8655-20-1-7027. Co-funded by the European Union (ERC, SAMPDE, 101041040). Views and opinions expressed are however those of the authors only and do not necessarily reflect those of the European Union or the European Research Council. Neither the European Union nor the granting authority can be held responsible for them. The authors are members of the ``Gruppo Nazionale per l’Analisi Matematica, la Probabilità e le loro Applicazioni'', of the ``Istituto Nazionale di Alta Matematica''.

%\section*{Declarations}
%
%\noindent
%\textbf{Conflict of interest.} The authors declare that they have no conflict of interest.

\bibliographystyle{abbrv}
\bibliography{biblio}

\appendix

\section{Appendix}

\subsection{Discrete strided convolutions}
\label{proof_adjoint}
A fractional-strided convolution $\Psi \colon (\R^{s \alpha})^{c_{\text{in}}} \to (\R^{\alpha})^{c_{\text{out}}}$ with stride $s$ such that $\alpha s,s^{-1} \in \N^*$, $c_{\text{in}}$ input channels and $c_{\text{out}}$ output channels is defined by
\begin{equation*}
    (\Psi x)_k := \sum_{i=1}^{c_{\text{in}}} x_i \ast_{s} t_{i,k} + {b}_k,\qquad k = 1,...,c_{\text{out}},
\end{equation*}
where $t_{i,k} \in \mathbb{R}^{\alpha}$  are the convolutional filters and ${b}_k \in \mathbb{R}^{\alpha}$  are the bias terms, for $i=1,...,c_{\text{in}}$ and $k=1,...,c_{\text{out}}$.
The operator $\ast_s$ is defined in \eqref{deconv_discr} as
\begin{equation}
(x \ast_{s} t) (n) := \sum_{m \in \Z} x(m) \hspace{0.1cm} t(n-s^{-1}m),
\end{equation}
where we extend the signals $x$ and $t$ to finitely supported sequences by defining them  zero outside their supports, i.e.\ $x,t \in c_{00}(\Z)$, where $c_{00}(\Z)$ is the space of sequences with finitely many nonzero elements.
We can rewrite \eqref{deconv_discr} in the following way:
\begin{equation}
\label{deconv_discr_2}
(x \ast_{s} t) (n) = (x \ast t_r)(k),
\end{equation}
where $n = s^{-1}k +r$ with $r \in \{ 0,...,s^{-1}-1 \}$, $k \in \mathbb{Z}$ and $t_r = t(s^{-1} \cdot +r)$ for every $r=0,...,s^{-1}-1$. 
The symbol $\ast$ represents the discrete convolution
\begin{equation}
\label{ast}
x \ast t := \displaystyle \sum_{m \in \mathbb{Z}} x(m) \hspace{0.1cm} t(\cdot - m), \qquad  x,t \in c_{00}(\Z).
\end{equation}
The output $x \ast_{s} t$ belongs to $c_{00}(\Z)$.
We motivate \eqref{deconv_discr} by taking the adjoint of the strided convolution with stride $s^{-1} \in \N^*$, as we explain below.

Equation~\eqref{deconv_discr_2} is useful to interpret Hypothesis~\ref{hyp_filtri} on  the convolutional filters (Section~\ref{sec:injectivity}).
We observe that in our case the convolution is well defined since the signals we consider have a finite number of non-zero entries. However, in general, it is enough to require that $x \in \ell^p$ and $t \in \ell^q$ with $\frac{1}{p}+\frac{1}{q} = 1$ to obtain a well-defined discrete convolution.

We now want to justify \eqref{deconv_discr}. We compute the adjoint of the convolutional operator $A_{s^{-1},t}$ with stride $s^{-1} \in \N^*$ and filter $t$, which is defined as $A_{s^{-1},t} \hspace{0.1cm} x = x \ast_{s^{-1}} t$, where
\begin{equation}
    \label{conv_discr}
     (x \ast_{s^{-1}} t) (n) := (x \ast t)(s^{-1} n) = \sum_{m \in \Z} x(m) \hspace{0.1cm} t(s^{-1}n-m).
\end{equation}
As before, the signals $x$ and $t$ are seen as elements of $c_{00}(\Z)$ by extending them to zero outside their supports and the symbol $\ast$ denotes the discrete convolution defined in \eqref{ast}.

The adjoint of $A_{s^{-1},t}$, $A^*_{s^{-1},t}$, satisfies
\begin{equation}
\label{adjoint}
    \langle A^*_{s^{-1},t} \hspace{0.1cm} y, x  \rangle_2 = \langle y, A_{s^{-1},t} \hspace{0.1cm} x  \rangle_2,
\end{equation}
where $\langle \cdot, \cdot \rangle_2$ is the scalar product on $\ell^2$. Using \eqref{adjoint}, we find that
\begin{equation}
\label{A^*}
    (A^*_{s^{-1},t} \hspace{0.1cm} y)(n) = \sum_{m \in \Z} y(m) \hspace{0.1cm} t(s^{-1}m-n).
\end{equation}
However, in order to be consistent with the definition in \eqref{conv_discr} when $s=1$, we do not define the fractionally-strided convolution as the adjoint given by \eqref{A^*}, but as in \eqref{deconv_discr}.

Figure~\ref{fig_strides} presents a graphical illustration of three examples of strided convolutions with different strides: $s=1$ in Figure~\ref{stride_1}, $s=2$ in Figure~\ref{stride_2}, and $s=\frac{1}{2}$ in Figure~\ref{stride_1_2}. The input vector $x$ and the filter $t$ have a finite number of non-zero entries indicated with yellow and orange squares/rectangles, respectively, and the output $x \ast_s t$ has a finite number of non-zero entries indicated with red squares/rectangles. For simplicity, we identify the infinite vectors in $\R^{\Z}$ with vectors in $\R^{N}$ where $N$ is the number of their non-zero entries. Given the illustrative purpose of these examples, for simplicity we ignore boundary effects. The signals' sizes are:
\begin{enumerate}[label=(\alph*)]
\item $s = 1$, input vector $x \in \mathbb{R}^8$, output vector $x \ast_1 t \in \mathbb{R}^8$;
\item $s = 2$, input vector $x \in \mathbb{R}^8$, output vector $x \ast_2 t \in \mathbb{R}^4$;
\item $s = \frac{1}{2}$, input vector $x \in \mathbb{R}^8$, output vector $x \ast_{\frac{1}{2}} t \in \mathbb{R}^{16}$.
\end{enumerate}

\begin{figure}
\centering
\subfloat[Discrete convolution with stride $s = 1$ between an input vector $x \in \mathbb{R}^{8}$ (in yellow) and a filter $t \in \mathbb{R}^{3}$ (in orange), which gives as output a vector $x \ast_{1} t \in \mathbb{R}^{8}$ (in red). The input and output sizes are the same.]{\label{stride_1}\includegraphics[width=0.40\textwidth]{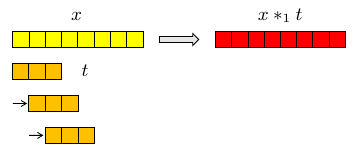}} 
\hfill
\subfloat[Discrete convolution with stride $s = 2$ between an input vector $x \in \mathbb{R}^{8}$ (in yellow) and a filter $t \in \mathbb{R}^{3}$ (in orange), which gives as output a vector $x \ast_{2} t \in \mathbb{R}^{4}$ (in red). The output size is half the input size.]{\label{stride_2}\includegraphics[width=0.40\textwidth]{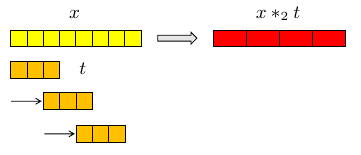}} \\
\subfloat[Intuition of the convolution with stride $s = \frac{1}{2}$ between an input vector $x \in \mathbb{R}^{8}$ (in yellow) and a filter $t \in \mathbb{R}^{3}$ (in orange), which gives as output a vector $x \ast_{\frac{1}{2}} t \in \mathbb{R}^{16}$ (in red). The output size is twice the input size.]{\label{stride_1_2}\includegraphics[width=0.50\textwidth]{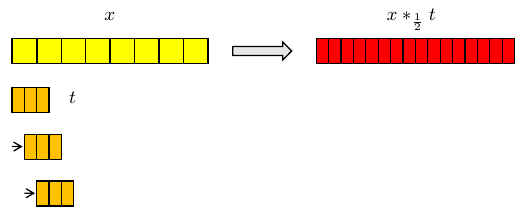}}
\caption{Discrete convolutions with strides $s = 1$ and $s=2$ and intuitive interpretation of convolution with stride $s = \frac 1 2$.}
\label{fig_strides}
\end{figure}

When the stride is an integer, equation~\eqref{conv_discr} describes what is represented in Figures~\ref{stride_1} and \ref{stride_2}. When the stride is $s = \frac{1}{2}$, as depicted in Figure~\ref{stride_1_2}, it is intuitive to consider a filter whose entries are half the size of the input ones. This is equivalent to choosing the filter in a space of higher resolution with respect to the space of the input signal. As a result, the output belongs to the same higher resolution space. For instance, the filter belongs to a space that is twice the resolution of the input space when the stride is $\frac{1}{2}$. This notion of resolution is coherent with the scale parameter used in the continuous setting.

In fact, Figure~\ref{stride_1_2} does not represent equation~\eqref{deconv_discr} exactly when $s=\frac12$. However the illustration is useful to model the fractional-strided convolution in the continuous setting.
A more precise illustration of the $\frac 1 2$-strided convolution of equation~\eqref{conv_discr} is presented in Figure~\ref{stride_1_2_real}.

\begin{figure}
\centering 
\includegraphics[width=\textwidth]{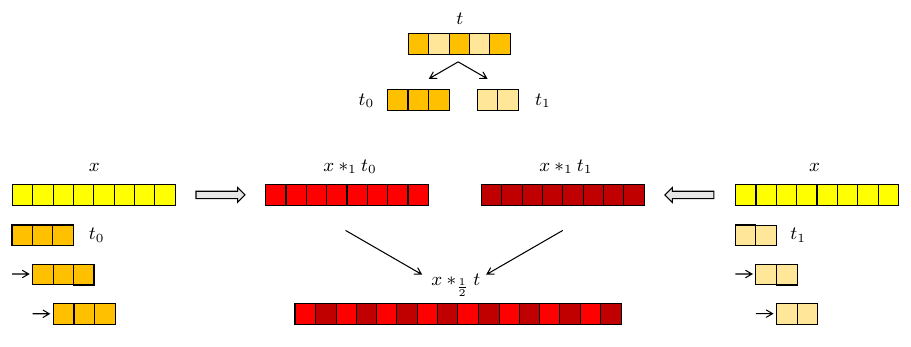} 
\caption{Discrete convolution with stride $s = \frac{1}{2}$ between an input vector $x \in \mathbb{R}^{8}$ (in yellow) and a filter $t \in \mathbb{R}^{3}$ (in orange), which gives as output a vector $ x \ast_{\frac{1}{2}} t \in \mathbb{R}^{16}$ (in red). The filter $t$ is first divided into two \textit{sub-filters}, $t_0$ and $t_1$, containing the odd and even entries, respectively. Then, a convolution with stride $1$ between the input vector and each \textit{sub-filter} is performed. Finally, the two output vectors are reassembled to form the final output vector.}
\label{stride_1_2_real}
\end{figure}

\subsection{Extensions}
\label{sec:extension}
In Section~\ref{sec:injectivity} we considered the following simplified framework:
\begin{enumerate}
    \item non-expansive convolutional layers;
    \item stride $s = \frac{1}{2}$ for each convolutional layer;
    \item one-dimensional signals.
\end{enumerate}
In the following subsections we extend our theory by weakening these assumptions.

\subsubsection{Expansive convolutional layers and arbitrary stride}
\label{sec:expansive_layer_and_arbitrary_stride}
In Section~\ref{sec:injectivity}, we considered the case where the stride is $s=\frac 1 2$ for each layer and where the number of channels scaled exactly by the factor $s$ at each layer, i.e.\ $c_{l} = s c_{l-1}$.
Here, we generalize Theorem~\ref{main_thm} to the case of  stride $s_{l} = \frac{1}{2^{\nu_{l}}}$, with $\nu_{l} \in \mathbb{N}$, for the $l$-th convolutional layer and by considering expansive layers, i.e.\ $c_{l} \geq s c_{l-1}$. In this case, $D^l$ are not necessarily square matrices, so we need to impose a condition on their rank.

\begin{hyp}
\label{hyp_filtri_0}
Let $\bar p\in\N$. For every $l=2,...,L$, the convolutional filters $t^{l}_{i,k} \in V_{j_{l}}$ of the $l$-th convolutional layer \eqref{psi_ell} satisfy    \begin{equation}
    t^{l}_{i,k} = \sum_{p=0}^{ \bar{p}} d^{l}_{p,i,k} \phi_{j_{l},p}, \qquad i=1,...,c_{l-1}, \qquad k=1,...,c_{l},
    \end{equation}
    where $d^{l}_{p,i,k} \in \R$, and the rank of ${D}^{l}$ is maximum, where ${D}^{l}$ is a $2^{\nu_l} c_{l} \times c_{l-1}$ matrix defined by
\begin{equation}
\label{tilde_d_p_i_k_0}
({D}^{l})_{i,k} :=
    \begin{cases}
      d^{l}_{0,i,k} & k = 1,...,c_{l}, \\
      d^{l}_{1,i,k-c_{l}} & k = c_{l}+1,...,2 c_{l}, \\
      d^{l}_{2,i,k-2c_{l}} & k = 2 c_{l}+1,...,3 c_{l}, \\
      \vdots \\
      d^{l}_{2^{\nu_l}-1,i,k-(2^{\nu_l}-1)c_{l}} & k = (2^{\nu_l} -1)c_{l}+1,...,2^{\nu_l} c_{l}.
    \end{cases}
\end{equation}
\end{hyp}

\begin{thm}
\label{cor_expansive_stride}
Let $L \in \mathbb{N}^*$ and $j_1 \in \mathbb{Z}$. Let $c_1,\dots,c_L\in \N^*$ and $\nu_2,\dots,\nu_L\in\N$ such that $c_L=1$, $c_{l-1}$ is divisible by $2^{\nu_{l}}$ and $c_{l} \geq \frac{c_{l-1}}{2^{\nu_{l}}}$. Set $j_{l} = j_{l-1} +\nu_{l}$ for every $l=2,...,L$. Let  $V_{j_{l}}$ be the spaces of an MRA and $t^{l}_{i,k} \in V_{j_{l}}$ for every $l=2,...,L$, $i=1,...,c_{l-1}$ and $k=1,...,c_{l}$. Let $\tilde{\Psi}_{l}$ and $\tilde{\sigma}_{l}$ be defined as in \eqref{tilde_psi} and \eqref{tilde_sigma}, respectively. Let $\Psi_1$ be defined as in \eqref{eq:Psi1}.
If Hypotheses~\ref{hyp_scaling_spaces},  \ref{hyp_sigma}, \ref{hyp_f_c} and \ref{hyp_filtri_0} are satisfied, then the generator $G$ defined in \eqref{continuous_generator} is injective.
\end{thm}
\begin{proof}[Sketch of the proof]
The proof of the injectivity of $\tilde{\sigma}$ is the same as in Theorem~\ref{main_thm}. Moreover, the injectivity of $\tilde{\Psi}$ is guaranteed by Hypothesis~\ref{hyp_filtri_0} by following the same argument used in the proof of Theorem~\ref{main_thm}.
\end{proof}

We observe that, thanks to equation~\eqref{stride_e_j}, choosing $j_{l} = j_{l-1} + \nu_{l}$ is equivalent to imposing that the stride of the $l$-th convolutional layer is $s_{l} = \frac{1}{2^{\nu_{l}}}$ with $\nu_{l} \in \mathbb{N}$.

\subsubsection{The two-dimensional case}
\label{sec:2D}

\paragraph{2D wavelet analysis}\label{sec:2dwave}
From \cite[Section $7.7$]{Mallat}, 
we recall the principal concepts of $2$D wavelet analysis. In $2$D, the scaling function spaces become $V_j \otimes V_j$ with $j \in \mathbb{Z}$, with orthonormal basis given by $\{ \phi_{j,(n_1,n_2)} \}_{(n_1,n_2) \in \mathbb{Z}^2}$, where
\begin{equation*}
    \phi_{j,(n_1,n_2)}(x_1,x_2) = \phi_{j,n_1}(x_1) \phi_{j,n_2}(x_2),
\end{equation*}
and $\{ \phi_{j,n} \}_{n \in \mathbb{Z}}$ is an orthonormal basis of $V_j$. We recall that $$V\otimes W=\overline{\spann\{f\otimes g:f\in V,g\in W\}},$$ where $(f\otimes g)(x)=f(x_1)g(x_2)$.
As in $1$D, the MRA properties of  Definition~\ref{MRA} hold:
\begin{enumerate}
    \item $\displaystyle \bigcup_{j \in \mathbb{Z}} (V_j\otimes V_j)$ is dense in $L^2(\mathbb{R}^2) \cong L^2(\R) \otimes L^2(\R)$ and $\displaystyle \bigcap_{j \in \mathbb{Z}} (V_j\otimes V_j) = \{ 0 \}$;
	\item $f \in V_j \otimes V_j$ if and only if $f(2^{-j} \cdot, 2^{-j} \cdot) \in V_0 \otimes V_0$;
	\item and $\{ \phi_{0,(n_1,n_2)} \}_{n_1,n_2 \in \mathbb{Z}}$ is an orthonormal basis of $V_0 \otimes V_0$.
\end{enumerate}
Moreover, $V_j \otimes V_j \subset V_{j+1} \otimes V_{j+1}$ for every $j \in \mathbb{Z}$, as in $1$D.

\paragraph{2D CGNNs}

We begin by specifying the structure of the nonlinearities, the fully connected layer and the convolutional ones in the $2$D setting.

The nonlinearities $\sigma_l$ act on functions in $L^2(\R^2)^{c_{l}}$ by pointwise evaluation:
\begin{equation*}
\sigma_l(f)(x) = \sigma_l(f(x)),\qquad \text{a.e.\ $x\in \mathbb{R}^2$}.
\end{equation*}

The fully connected layer is defined as $\Psi_1:=F\cdot +b$, where $F\colon\mathbb{R}^\Esse\to (V_{j_1} \otimes V_{j_1})^{c_1}$ is a linear map and $b\in (V_{j_1} \otimes V_{j_1})^{c_1}$.

The convolutional layers are
\begin{equation*}
  \bar\Psi_l = P_{(V_{j_{l}} \otimes V_{j_{l}})^{c_{l}}}\circ \Psi_{l} \circ P_{(V_{j_{l-1}} \otimes V_{j_{l-1}})^{c_{l-1}}},
\end{equation*}
where the convolution  $\Psi_l \colon L^2(\R^2)^{c_{l-1}}\to L^2(\R^2)^{c_{l}}$ is given by
\begin{equation}
\label{psi_ell_2}
  (\Psi_{l} (x))_k := \sum_{i=1}^{c_{l-1}} x_i * t^{l}_{i,k} + {b}^{l}_k,\qquad k = 1,...,c_{l},
\end{equation}
with filters $t^{l}_{i,k}\in (V_{j_{l}} \otimes V_{j_{l}})\cap L^1(\R^2)$ and biases ${b}^{l}_k\in V_{j_{l}} \otimes V_{j_{l}}$. 
In \eqref{psi_ell_2}, the symbol $\ast$ denotes the usual convolution $L^2(\R^2)\times L^1(\R^2)\to L^2(\R^2)$.
Note that, in analogy to the $1$D case, the convolution of a filter in $V_{j + \nu} \otimes V_{j + \nu}$ with a function in $V_j \otimes V_j$ produces a function that does not necessarily belong to $V_{j + \nu} \otimes V_{j + \nu}$. In view of identity~\eqref{stride_e_j}, this corresponds to a stride $s = (\frac{1}{2^{\nu}}, \frac{1}{2^{\nu}})$.

We can finally define the CGNN architecture in $2$D:
\begin{multline*}
G \colon \mathbb{R}^{\Esse} \xrightarrow[\text{f.c.}]{\Psi_1} (V_{j_1} \otimes V_{j_1})^{c_1} \xrightarrow[\text{nonlin.}]{\sigma_1} L^2(\R^2)^{c_1} \xrightarrow[\text{proj.}]{P_{(V_{j_1} \otimes V_{j_1})^{c_1}}} (V_{j_1} \otimes V_{j_1})^{c_1} \xrightarrow[\text{conv.}]{\Psi_2} L^2(\R^2)^{c_2} \\  \xrightarrow[\text{proj.}]{P_{(V_{j_2} \otimes V_{j_2})^{c_2}}} (V_{j_2} \otimes V_{j_2})^{c_2} 
\xrightarrow[\text{nonlin.}]{\sigma_2} L^2(\R^2)^{c_2}  \xrightarrow[\text{proj.}]{P_{(V_{j_2} \otimes V_{j_2})^{c_2}}} (V_{j_2} \otimes V_{j_2})^{c_2}  \xrightarrow[\text{conv.}]{\Psi_3}  ... \\ ... \xrightarrow[\text{conv.}]{\Psi_{L}} L^2(\R^2)  \xrightarrow[\text{proj.}]{P_{V_{j_L} \otimes V_{j_L}}} V_{j_L} \otimes V_{j_L}  
\xrightarrow[\text{nonlin.}]{\sigma_L} L^2(\R^2) \xrightarrow[\text{proj.}]{P_{V_{j_L} \otimes V_{j_L}}} V_{j_L} \otimes V_{j_L},
\end{multline*}
which can be summarized as
\begin{equation}
\label{continuous_generator_2D}
    G= \left( \comp_{l=L}^2 \tilde{\sigma}_{l} \circ \tilde\Psi_l\right)
    \circ \left(\tilde{\sigma}_1 \circ \Psi_1 \right),
\end{equation}
where
\begin{equation}
\label{tilde_psi_2D}
\tilde{\Psi}_{l} := P_{(V_{j_{l}} \otimes V_{j_{l}})^{c_{l}}} \circ \Psi_{l}\colon (V_{j_{l-1}} \otimes V_{j_{l-1}})^{c_{l-1}}\to (V_{j_{l}} \otimes V_{j_{l}})^{c_{l}}, \qquad l=2,...,L,  
\end{equation}
and
\begin{equation}
\label{tilde_sigma_2D}
\tilde{\sigma}_{l} := P_{(V_{j_{l}} \otimes V_{j_{l}})^{c_{l}}} \circ \sigma_l \colon (V_{j_{l}} \otimes V_{j_{l}})^{c_{l}} \to (V_{j_{l}} \otimes V_{j_{l}})^{c_{l}}, \qquad l=1,...,L.     
\end{equation}

For simplicity, we state our injectivity result using non-expansive convolutional layers and stride $(\frac{1}{2}, \frac{1}{2})$ for each convolutional layer. The result can be extended to arbitrary strides and expansive layers by adapting the arguments in \ref{sec:expansive_layer_and_arbitrary_stride}.

\begin{hyp}
\label{hyp_filtri_3}
Let $\bar{p} \in \N$. For every $l=2,...,L$, the convolutional filters $t^{l}_{i,k} \in V_{j_{l}} \otimes V_{j_{l}}$ of the $l$-th convolutional layer satisfy 
    \begin{equation*}
    t^{l}_{i,k} = \sum_{p_1=0}^{\bar{p}} \sum_{p_2=0}^{\bar{p}} d^{l}_{p,i,k} \phi_{j_{l},p},
    \end{equation*}
where $d^{l}_{p,i,k} \in \R$ and $p = (p_1,p_2)$, and $\det({D}^{l}) \neq 0$, where ${D}^{l}$ is a $\frac{c_1}{2^{2l-2}} \times \frac{c_1}{2^{2l-2}}$ matrix defined by
\begin{equation*}
    ({D}^{l})_{i,k} :=
    \begin{cases}
      d^{l}_{(0,0),i,k}, & k = 1,...,\frac{c_1}{2^2 2^{l-1}}, \\
      d^{l}_{(1,0),i,k-\frac{c_1}{2^2 2^{l-1}}}, & k = \frac{c_1}{2^2 2^{l-1}} +1 ,...,\frac{c_1}{2 2^{l-1}}, \\
      d^{l}_{(0,1),i,k-\frac{2 c_1}{2^2 2^{l-1}}}, & k = \frac{c_1}{2 2^{l-1}} + 1,...,\frac{3 c_1}{2^2 2^{l-1}}, \\
      d^{l}_{(1,1),i,k-\frac{3 c_1}{2^2 2^{l-1}}}, & k = \frac{3 c_1}{2^2 2^{l-1}}+ 1,...,\frac{c_1}{2^{l-1}}.
    \end{cases}
\end{equation*}
\end{hyp}

\begin{hyp}
\label{hyp_f_c_1}
We assume that
\begin{itemize}
    \item The linear function $F \colon \R^{\Esse} \to (V_{j_1})^{c_1} \otimes (V_{j_1})^{c_1}$ is injective;
    \item There exists $N \in \mathbb{N}$ such that $b \in (\spann\{ \phi_{j_1,n} \}_{n_1,n_2=-N}^{N})^{c_1}$ and \\$\Im(F) \subset (\spann\{ \phi_{j_1,n} \}_{n_1,n_2 = -N}^{N})^{c_1}$ with $n = (n_1,n_2)$. 
\end{itemize}
\end{hyp}

\begin{thm}
\label{cor_2D}
Let $L \in \mathbb{N}^*$ and $j_1 \in \mathbb{Z}$. Let $c_1 = 2^{2 L - 2}$, $c_{l} = \frac{c_1}{2^{2 l - 2}}$ and $j_{l} = j_1 + l - 1$ for every $l=1,...,L$. Let $V_{j_{l}}$ be scaling function spaces arising from an MRA, as defined in $\S$\ref{sec:2dwave}, and $t^{l}_{i,k} \in V_{j_{l}} \otimes V_{j_{l}}$ for every $l=2,...,L$, $i=1,...,c_{l-1}$ and $k=1,...,c_{l}$. Let $\tilde{\Psi}_{l}$ and $\tilde{\sigma}_{l}$ be defined as in \eqref{tilde_psi_2D} and \eqref{tilde_sigma_2D}, respectively. Let $F \colon \R^{\Esse} \to (V_{j_1})^{c_1} \otimes (V_{j_1})^{c_1}$ be a linear map and $b \in (V_{j_1})^{c_1} \otimes (V_{j_1})^{c_1}$.
If Hypotheses~\ref{hyp_scaling_spaces},  \ref{hyp_sigma}, \ref{hyp_filtri_3} and \ref{hyp_f_c_1} are satisfied, then the generator $G$ defined in \eqref{continuous_generator_2D} is injective.
\end{thm}
\begin{proof}[Sketch of the proof]
The proof follows from the same arguments of the proof of Theorem~\ref{main_thm}, by considering $n = (n_1,n_2) , m = (m_1,m_2) , p = (p_1,p_2) \in \mathbb{Z}^2$ and $\xi = (\xi_1,\xi_2) \in [0,1]^2$. 
\end{proof}

\end{document}